\documentclass[table,xcdraw, 11.5pt]{article}
\usepackage[geometry]{ifsym}

\usepackage{supertabular,booktabs}
\usepackage{pgf}
\usepackage{amsmath,amssymb}
\usepackage{mathrsfs}
\usepackage{bbm}
\usepackage{verbatim}
\usepackage{epsfig}
\usepackage{graphicx}
\usepackage{color}
\usepackage{import}
\usepackage{bussproofs}
\usepackage{qtree}
\usepackage{tree-dvips}
\usepackage{lingmacros}
\usepackage{amsfonts}
\usepackage{mathdots}
\usepackage{amssymb}
\usepackage{pifont}
\usepackage[safe]{tipa}
\usepackage{newlfont}
\usepackage{latexsym}
\usepackage{multirow}
\usepackage{array,cmll}
\usepackage{textcomp}
\usepackage{accents}
\usepackage{rotating}
\usepackage{logreq}

\usepackage{hyperref}

  \usepackage{amsmath}
  \usepackage{amsthm}
  \usepackage{latexsym}
  \usepackage{amscd}
  \usepackage{amssymb}
  \usepackage{bussproofs}
  \usepackage{txfonts}

  \usepackage{tikz}
  \usepackage[position=top]{subfig}

  \usepackage{leqno}
  \usepackage{txfonts}
  \usepackage{tikz}

\usepackage{setspace}

\EnableBpAbbreviations
\usepackage{setspace}
\usepackage{url}
\usepackage{algpseudocode,algorithm}
\usepackage{authblk}
\usepackage{amsmath}
\usepackage{amsfonts}
\usepackage{longtable}
\usepackage{color}
\usepackage[T1]{fontenc}
\usepackage[all]{xy}
\usepackage{amssymb}
\usepackage{wrapfig}
\usepackage[left=3cm,top=3cm,right=3cm,bottom=2cm]{geometry}
\usepackage{graphicx}
\usepackage{cmll}
\usepackage{etex}
\usepackage{graphics}
\usepackage{color}
\usepackage{cancel}
\usepackage{stmaryrd}
\usepackage{amsmath}
\usepackage{amsthm}
\usepackage{amssymb,stmaryrd}
\usepackage{proof}
\usepackage{verbatim,cmll}
\usepackage{setspace}
\usepackage{lscape}
\usepackage{latexsym,relsize}
\usepackage{amscd}
\usepackage{multicol}
\usepackage{bussproofs}
\usepackage[position=top]{subfig}
\usepackage{leqno}
\usepackage{tikz}
\usetikzlibrary{arrows,decorations.pathmorphing,backgrounds,positioning,fit,petri}
\usepackage{multicol}
\usetikzlibrary{arrows}
\usetikzlibrary{matrix}
\usetikzlibrary{patterns}
\usetikzlibrary{shapes}

\usepackage[utf8]{inputenc}

\def\aol{\rule[0.5865ex]{1.38ex}{0.1ex}}
\def\pdla{\mbox{\rotatebox[origin=c]{180}{$\,>\mkern-8mu\raisebox{-0.065ex}{\aol}\,$}}}
\def\pdra{\mbox{$\,>\mkern-8mu\raisebox{-0.065ex}{\aol}\,$}}

\theoremstyle{plain}
\newtheorem{thm}{Theorem}

\newtheorem{cor}[thm]{Corollary}
\newtheorem{prop}[thm]{Proposition}

\newtheorem{lemma}[thm]{Lemma}

\theoremstyle{definition}

\newtheorem{definition}[thm]{Definition}

\newtheorem{remark}[thm]{Remark}

\title{Flexible categorization using formal concept analysis and Dempster-Shafer theory}
\author[5,6]{Marcel Boersma}
\author[1,6]{Krishna Manoorkar\footnote{Krishna Manoorkar is supported by the NWO grant KIVI.2019.001 awarded to Alessandra Palmigiano.}}
\author[1,2]{Alessandra Palmigiano}
\author[1,6]{Mattia Panettiere}

\author[1]{Apostolos Tzimoulis}
\author[3,4]{Nachoem Wijnberg} 

\affil[1]{School of Business and Economics, Vrije Universiteit, Amsterdam, The Netherlands}
\affil[2]{Department of Mathematics and Applied Mathematics, University of Johannesburg, South Africa}
\affil[3]{College of Business and Economics, University of Johannesburg, South Africa}
\affil[4]{Faculty of Economics and Business, University of Amsterdam, The Netherlands}
\affil[5]{Computational Science Lab, University of Amsterdam, Amsterdam, The Netherlands}
\affil[6]{KPMG, Amstelveen, The Netherlands}

\date{}

\begin{document}
\maketitle
%
%
%

\newcommand{\bel}{\mathrm{bel}}
\newcommand{\pl}{\mathrm{pl}}
\newcommand{\mass}{{m}}
\newcommand{\redbf}[1]{\textcolor{red}{\textbf{#1}}}
\newcommand{\red}[1]{\textcolor{red}{#1}}
\newcommand{\bluebf}[1]{\textcolor{blue}{\textbf{#1}}}
\newcommand{\blue}[1]{\textcolor{blue}{#1}}
\newcommand{\marginred}[1]{\marginpar{\raggedright\tiny{\red{#1}}}}
\newcommand{\marginredbf}[1]{\marginpar{\raggedright\tiny{\redbf{#1}}}}
\newcommand{\redfootnote}[1]{\footnote{\redbf{#1}}}

\newcommand{\sabine}[1]{\textcolor{red}{\textbf{#1}}}
\newcommand{\willem}[1]{\textcolor{purple}{\textbf{#1}}}

\renewcommand{\P}{\mathcal{P}}
\newcommand{\C}{\mathsf{C}}
\renewcommand{\c}{\mathsf{c}}
\renewcommand{\d}{\mathsf{d}}
\newcommand{\up}[1]{#1^{\uparrow}}
\newcommand{\down}[1]{#1^{\downarrow}}
\newcommand{\ud}[1]{#1^{\uparrow\downarrow}}
\newcommand{\du}[1]{#1^{\downarrow\uparrow}}
\newcommand{\D}{\mathsf{D}}
\newcommand{\I}{\mathsf{I}}
\newcommand{\q}{\mathsf{q}}
\renewcommand{\u}{\mathsf{u}}
\renewcommand{\L}{\mathbb{L}}
\newcommand{\F}{\mathbb{F}}
\newcommand{\osigma}{\overline{\sigma}}
\renewcommand{\emph}{\textbf}
\newcommand{\ovR}[1]{\overline{#1}^R}
\newcommand{\ovS}[1]{\overline{#1}^S}
\newcommand{\op}{\mathbf{op}}
\newcommand{\mb}{\mathbb}
\newcommand{\mc}{\mathcal}

\newcommand{\Prop}{\mathsf{Prop}}
\newcommand{\Nom}{\mathsf{Nom}}
\newcommand{\Cnom}{\mathsf{Cnom}}

\newcommand{\low}{\mathsf{l}}
\newcommand{\cl}{\mathsf{c}}

\newcommand{\jty}{J^{\infty}}
\newcommand{\mty}{M^{\infty}}

\newcommand{\nomi}{\mathbf{i}}
\newcommand{\nomj}{\mathbf{j}}
\newcommand{\nomk}{\mathbf{k}}
\newcommand{\noma}{\mathbf{a}}

\newcommand{\cnomm}{\mathbf{m}}
\newcommand{\cnomn}{\mathbf{n}}
\newcommand{\cnomx}{\mathbf{x}}

\newcommand{\lc}{\mathbf{c}}
\newcommand{\Leq}{\mathbf{s}}
\newcommand{\nLeq}{\mathbf{n}}

\newcommand{\Cc}{\mathbb{C}}

\newcommand{\p}{\mathcal{P}}

\newcommand{\marginnote}[1]{\marginpar{\raggedright\tiny{#1}}} 

\renewcommand{\L}{\mathcal{L}}
\newcommand{\Lb}{\mathcal{L}_\Box}
\newcommand{\Lmu}{\mathcal{L}_C}
\newcommand{\Ag}{\mathsf{Ag}}
\newcommand{\Cat}{\mathsf{Cat}}
\newcommand{\Ca}{\mathsf{C}_a}
\newcommand{\Cka}{\mathsf{C}_{k,a}}

\newcommand{\nomb}{\mathbf{b}}
\newcommand{\nomx}{\mathbf{x}}
\newcommand{\val}[1]{[\![{#1}]\!]}
\newcommand{\descr}[1]{(\![{#1}]\!)}
\renewcommand{\phi}{\varphi}

\newcommand{\commment}[1]{}


\def\andol{\rule[-0.4563ex]{1.38ex}{0.1ex}}
\def\AOL{\rule[0.65ex]{1.45ex}{0.1ex}}
\def\aol{\rule[0.5865ex]{1.38ex}{0.1ex}}


\def\pRA{\mbox{$\,\,{\AOL{\mkern-0.4mu{\rotatebox[origin=c]{-90}{\raisebox{0.12ex}{$\pAND$}}}}}\,\,$}}

\def\pDRA{\mbox{$\,\,{\rotatebox[origin=c]{-90}{\raisebox{0.12ex}{$\pAND$}}{\mkern-1mu\AOL}}\,\,$}}
\def\pdra{\mbox{$\,>\mkern-8mu\raisebox{-0.065ex}{\aol}\,$}}
\def\pdla{\mbox{\rotatebox[origin=c]{180}{$\,>\mkern-8mu\raisebox{-0.065ex}{\aol}\,$}}}

\def\gSCRA{\mbox{$\succ$}}
\def\gSCLA{\mbox{$\prec$}}

\def\gscra{\mbox{$\,\raisebox{-0.065ex}{\aol}\mkern-5.5mu\raisebox{0.155ex}{${\scriptstyle{\succ}}$}\,$}}
\def\gscla{\mbox{\rotatebox[origin=c]{180}{$\,\raisebox{-0.065ex}{\aol}\mkern-5.5mu\raisebox{0.155ex}{${\scriptstyle{\succ}}$}\,$}}}

\def\gscdra{\mbox{$\,\raisebox{0.155ex}{${\scriptstyle{\succ}}$}\mkern-3.55mu\raisebox{-0.065ex}{\aol}\,$}}
\def\gscdla{\mbox{\rotatebox[origin=c]{180}{$\,\raisebox{0.155ex}{${\scriptstyle{\succ}}$}\mkern-3.55mu\raisebox{-0.065ex}{\aol}\,$}}}



\def\mANDORatom#1{\hbox{\hbox to 0pt{$#1\TriangleUp$\hss}$#1\TriangleDown$}}
\newcommand*{\mANDOR}{\mathrel{\mathchoice{\mANDORatom\displaystyle}
                                 {\mANDORatom\textstyle}
                                 {\mANDORatom\scriptstyle}
                                 {\mANDORatom\scriptscriptstyle}}}

\newcommand{\mcAND}{%
\mathrel{\ooalign{\raisebox{-0.39ex}{$\mbox{\TriangleUp}$}\cr\kern4.2pt{\raisebox{-0.13ex}{$\cdot$}}}}}
\newcommand{\mcand}{%
\mathrel{\ooalign{$\vartriangle$\cr\kern1.99pt{\raisebox{-0.17ex}{$\cdot$}}}}}

\newcommand{\mAND}{\raisebox{-0.39ex}{\mbox{\,\TriangleUp\,}}}
\newcommand{\mand}{\vartriangle}



\newcommand{\nAND}{%
\mathrel{\ooalign{$\mbox{\TriangleUp}$\cr\kern0pt$\mbox{\rotatebox[origin=c]{180}{\TriangleUp}}$}}}

\newcommand{\nand}{%
\mathrel{\ooalign{$\vartriangle$\cr\kern0pt$\triangledown$}}}

\newcommand{\mcBAND}{%
\mathrel{\ooalign{\raisebox{-0.39ex}{$\mbox{\FilledTriangleUp}$}\cr\kern4.2pt{\raisebox{-0.13ex}{${\color{white}\cdot}$}}}}}
\def\mBAND{\raisebox{-0.39ex}{\mbox{\,\FilledTriangleUp\,}}}
\def\mband{\mbox{$\mkern+2mu\blacktriangle\mkern+2mu$}}

\newcommand{\mcband}{%
\mathrel{\ooalign{$\blacktriangle$\cr\kern1.99pt{\raisebox{-0.17ex}{${\color{white}\cdot}$}}}}}

\def\mOR{\mbox{\,\rotatebox[origin=c]{-180}{\TriangleUp}\,}}
\def\mor{\mbox{$\mkern+2mu\triangledown\mkern+2mu$}}

\def\mBOR{\mbox{\,\FilledTriangleDown\,}}
\def\mbor{\mbox{$\blacktriangledown$}}

\def\mRAline{\mbox{\,\raisebox{0.43ex}{\AOL}$\mkern-2.6mu$\rotatebox[origin=c]{-90}{\TriangleUp}\,}}

\newcommand{\mcRA}{%
\mathrel{\ooalign{
                  \raisebox{-0.3ex}{$\rotatebox[origin=c]{-90}{$\mbox{{\TriangleUp}}$}$}
                                                                            \cr\kern2.7pt{\raisebox{0.2ex}{$\cdot\mkern1.3mu$}}}}}

\def\DRA{\mbox{\,\rotatebox[origin=c]{3.9999}{\TriangleRight}$\mkern-2.6mu$\raisebox{0.43ex}{\AOL}\,}}
\def\mRA{\mbox{\,\raisebox{-0.39ex}{\rotatebox[origin=c]{-90}{\TriangleUp}}\,}}
\newcommand{\mra}{\mbox{$\,-{\mkern-3mu\vartriangleright}\,$}}
\newcommand{\mbra}{\mbox{$\,-{\mkern-3mu\blacktriangleright}\,$}}

\newcommand{\mla}{\mbox{$\,{\vartriangleleft\mkern-8mu-}\,$}}

\def\mbla{\mbox{$\,\blacktriangleleft{\mkern-8mu-}\,$}}

\def\mbdra{\mbox{$\,\blacktriangleright{\mkern-8mu-}\,$}}

\newcommand{\mcra}{%
\mathrel{\ooalign{$\,{\vartriangleright\,}$\cr\kern3pt{\raisebox{0ex}{$\cdot$}}}}}

\newcommand{\mcraline}{%
-{\mkern-6mu{\mathrel{\ooalign{$\,{\vartriangleright\,}$\cr\kern3pt{\raisebox{0ex}{$\cdot$}}}}}}}

\def\mdra{\mbox{$\,\vartriangleright{\mkern-8mu-}\,$}}
\def\nra{\mbox{$\,\vartriangleright\,$}}

\newcommand{\mdraline}{%
{\mathrel{\ooalign{$\,{\vartriangleright\,}$\cr\kern3pt{\raisebox{0ex}{$\cdot$}}}}}{\mkern-6mu}-}

\newcommand{\cra}{%
\mathrel{\ooalign{$\,-{\mkern-3mu\vartriangleright\,}$\cr\kern8pt{\raisebox{0ex}{$\cdot$}}}}}
\def\mdra{\mbox{$\,\vartriangleright{\mkern-8mu-}\,$}}
\def\nra{\mbox{$\,\vartriangleright\,$}}

\def\mSRA{\mbox{\,\raisebox{0.43ex}{$\thicksim\mkern-1.3mu$}\rotatebox[origin=c]{3.9999}{\TriangleRight}\,}}
\def\nSRA{\mbox{\,\rotatebox[origin=c]{-90}{\TriangleUp}\,}}
\def\msra{\mbox{$\,\sim{\mkern-8mu\vartriangleright}\,$}}

\def\mBRA{\mbox{\,\raisebox{-0.39ex}{\rotatebox[origin=c]{-90}{\FilledTriangleUp}}\,}}
\newcommand{\mcBRA}{%
\mathrel{\ooalign{
                  \raisebox{-0.3ex}{$\rotatebox[origin=c]{-90}{$\mbox{\FilledTriangleUp}$}$}
                                                                            \cr\kern2.7pt{\raisebox{0.2ex}{${\color{white}\cdot}$}}}}}

\def\mBDRA{\mbox{\,\FilledTriangleRight\raisebox{0.43ex}{\AOL}}\,}
\def\BRA{\mbox{\,\rotatebox[origin=c]{-90}{\FilledTriangleUp}\,}}
\newcommand{\mcbra}{%
\mathrel{\ooalign{$\,-{\mkern-3mu\blacktriangleright\,}$\cr\kern8pt{\raisebox{0ex}{$\cdot$}}}}}
\def\mbdra{\mbox{$\,\blacktriangleright{\mkern-8mu-}\,$}}

\def\mLA{\mbox{\,\raisebox{-0.39ex}{\rotatebox[origin=c]{90}{\TriangleUp}}\,}}
\newcommand{\mcLA}{%
\mathrel{\ooalign{
                  \raisebox{-0.3ex}{$\rotatebox[origin=c]{90}{$\mbox{\TriangleUp}$}$}
                                                                                     \cr\kern5.5pt{\raisebox{0.2ex}{$\cdot$}}
                                                                                                                              }}}
\def\nLA{\mbox{\,\rotatebox[origin=c]{-3.9999}{\TriangleLeft}\,}}
\def\la{\mbox{$\,\vartriangleleft{\mkern-8mu-}\,$}}

\def\nla{\mbox{$\,\vartriangleleft\,$}}

\newcommand{\mcla}{%
\mathrel{\ooalign{$\,{\vartriangleleft\,}$\cr\kern5pt{\raisebox{0ex}{$\cdot$}}}}}

\newcommand{\mclaline}{%
-{\mkern-6mu{\mathrel{\ooalign{$\,{\vartriangleleft\,}$\cr\kern5pt{\raisebox{0ex}{$\cdot$}}}}}}}

\def\mdla{\mbox{$\,-{\mkern-3mu\vartriangleleft}\,$}}
\def\nla{\mbox{$\,\vartriangleleft\,$}}

\def\mSLA{\mbox{\,\raisebox{-0.4ex}{\rotatebox[origin=c]{90}{\TriangleUp}}\raisebox{0.20ex}{$\mkern-2mu\thicksim$}\,}}

\def\SLA{\mbox{\,\rotatebox[origin=c]{90}{\TriangleUp}\raisebox{0.43ex}{$\mkern-2.6mu\thicksim$}\,}}
\def\SDLA{\mbox{\,\raisebox{0.43ex}{$\thicksim\mkern-3.8mu$}\rotatebox[origin=c]{90}{\TriangleUp}\,}}

\def\SBLA{\mbox{\,\rotatebox[origin=c]{90}{\FilledTriangleUp}\raisebox{0.43ex}{$\mkern-2.6mu\thicksim$}\,}}
\def\SDBLA{\mbox{\,\raisebox{0.43ex}{$\thicksim\mkern-3.8mu$}\rotatebox[origin=c]{90}{\FilledTriangleUp}\,}}

\def\BRA{\mbox{\,$\raisebox{0.43ex}{\AOL}\mkern-2.6mu$\rotatebox[origin=c]{-90}{\FilledTriangleUp}\raisebox{0.43ex}\,}}
\def\DBRA{\mbox{\,\rotatebox[origin=c]{-90}{\FilledTriangleUp}\raisebox{0.43ex}{$\mkern-3.8mu\AOL$}\,}}

\def\nSLA{\mbox{\,\rotatebox[origin=c]{-3.9999}{\TriangleLeft}\,}}
\def\msla{\mbox{$\,\vartriangleleft{\mkern-8mu\sim}\,$}}
\def\msdla{\mbox{$\,{\sim\mkern-4.5mu\vartriangleleft}\,$}}

\def\mBLA{\mbox{\,\raisebox{-0.39ex}{\rotatebox[origin=c]{90}{\FilledTriangleUp}}\,}}
\newcommand{\mcBLA}{%
\mathrel{\ooalign{
                  \raisebox{-0.3ex}{$\rotatebox[origin=c]{90}{$\mbox{\FilledTriangleUp}$}$}
                                                                                     \cr\kern5.5pt{\raisebox{0.2ex}{${\color{white}\cdot}$}}
                                                                                                                                            }}}

\def\BLA{\mbox{\,\rotatebox[origin=c]{90}{\FilledTriangleUp}\,}}
\def\bla{\mbox{$\,\blacktriangleleft{\mkern-8mu-}\,$}}
\def\nbla{\mbox{$\,\blacktriangleleft\,$}}

\def\mBDLA{\mbox{\,\raisebox{0.43ex}{\AOL}{\FilledTriangleLeft}\,}}

\def\mSBRA{\mbox{\,\raisebox{0.43ex}{$\thicksim\mkern-1.3mu$}\FilledTriangleRight}\,}
\def\nSBRA{\mbox{\,\rotatebox[origin=c]{-90}{\FilledTriangleUp}\,}}
\def\msbra{\mbox{$\,\sim{\mkern-8mu\blacktriangleright}\,$}}

\def\mSBLA{\mbox{\,\raisebox{-0.4ex}{\rotatebox[origin=c]{90}{\FilledTriangleUp}}\raisebox{0.20ex}{$\mkern-2mu\thicksim$}\,}}
\def\nSBLA{\mbox{\,\rotatebox[origin=c]{90}{\FilledTriangleUp}\,}}
\def\msbla{\mbox{$\,\blacktriangleleft{\mkern-8mu\sim}\,$}}
\def\msbdla{\mbox{$\,{\sim\mkern-4.5mu\blacktriangleleft}\,$}}



  %
  \numberwithin{equation}{section}



\def\odotRA{\mathrel\odot\joinrel\rightarrow}
\def\astRA{\mathrel\ast\!\joinrel\rightarrow}
\def\astLA{\mathrel\leftarrow\!\!\!{\joinrel\!\ast}}
\def\astSRA{\mathrel\ast\!\!\!\joinrel{>}\,}

\newcommand{\f}{\overline}

\newcommand{\ls}{\lbrack}
\newcommand{\rs}{\rbrack}
\newcommand{\rc}{\rangle}

\newcommand{\Wbox}{\Box}
\newcommand{\Wdia}{\Diamond}
\newcommand{\Bbox}{\blacksquare}
\newcommand{\Bdia}{\Diamondblack}

\newcommand{\Mod}{\, \Delta \,}



\def\conRA{\mbox{$\,\,{\rotatebox[origin=c]{-90}{\raisebox{0.12ex}{$\pAND$}}}\,\,$}}
\def\conLA{\mbox{$\,\,{\rotatebox[origin=c]{90}{\raisebox{0.12ex}{$\pAND$}}}\,\,$}}

\def\conDRA{\mbox{$\,\,{\rotatebox[origin=c]{90}{\raisebox{0.12ex}{$\pOR$}}}\,\,$}}
\def\conDLA{\mbox{$\,\,{\rotatebox[origin=c]{-90}{\raisebox{0.12ex}{$\pOR$}}}\,\,$}}

\newcommand{\pneg}{\neg}
\newcommand{\pNEG}{\mbox{\boldmath{${\neg\,}$}}}
\newcommand{\pNEGL}{\mbox{{\boldmath{$\neg$}}$_L$\,}}
\newcommand{\pNEGR}{\mbox{{\boldmath{$\neg$}}$_R$\,}}
\newcommand{\pand}{\wedge}
\newcommand{\pAND}{\mbox{$\,\bigwedge\,$}}
\newcommand{\por}{\vee}
\newcommand{\pOR}{\mbox{$\,\bigvee\,$}}
\newcommand{\pra}{\rightarrow}
\newcommand{\plra}{\leftrightarrow}
\newcommand{\pla}{\leftarrow}
\def\pTOP{{_{\textrm{p\!\!}}\textrm{T}}}
\def\ptop{{_{\textrm{p\!\!\!}}\top}}
\def\pBOT{\textrm{\rotatebox[origin=c]{180}{T}}^{\textrm{\!p}}}
\def\pbot{\bot^{\textrm{\!\!p}}}
\def\pneg{\neg}

%
\newcommand{\gI}{%
\mathrel{\ooalign{$\mbox{T}$\cr\kern0pt$\mbox{\rotatebox[origin=c]{180}{T}}$}}}
\newcommand{\gtop}{\tau}
\newcommand{\gbot}{\rotatebox[origin=c]{180}{$\tau$}}
\newcommand{\dwn}{\rotatebox[origin=c]{180}{$\wn$}}
\newcommand{\DWN}{\rotatebox[origin=c]{180}{{\boldmath{$\wn$}}}}
\newcommand{\WN}{\rotatebox[origin=c]{}{{\boldmath{$\wn$}}}}
\def\gbwn{\mbox{$\,\accentset{\wn}{\cdot}\,$}}
\newcommand{\gBWN}{\,\accentset{\wn}{\bullet}\,}
\newcommand{\gWN}{\,\accentset{\wn}{\circ}\,}
\newcommand{\apdla}{\pdla'}

\newcommand{\gAP}{\bullet}
\newcommand{\gap}{\cdot}
\newcommand{\gC}{\between}
\newcommand{\gsc}{\mbox{$\,;\,$}}
\newcommand{\gSC}{\mbox{\boldmath{{\Large{$\,;\,$}}}}}
\newcommand{\gand}{\cap}
\newcommand{\gAND}{\mbox{$\,\bigcap\,$}}
\newcommand{\gor}{\cup}
\newcommand{\gOR}{\,\bigcup\,}
\def\gra{\mbox{$\,{\raisebox{0.065ex}{\aol}{\mkern-1.4mu{\rotatebox[origin=c]{-90}{\raisebox{0.12ex}{$\gand$}}}}}\,$}}
\def\gRA{\mbox{$\,\,{{\rotatebox[origin=c]{-90}{\raisebox{0.12ex}{$\gAND$}}}}\,\,$}}

\def\gla{\mbox{\rotatebox[origin=c]{-180}{$\,{\raisebox{0.065ex}{\aol}{\mkern-1.4mu{\rotatebox[origin=c]{-90}{\raisebox{0.12ex}{$\gand$}}}}}\,$}}}
\def\gLA{\mbox{\rotatebox[origin=c]{180}{$\,\,{{\rotatebox[origin=c]{-90}{\raisebox{0.12ex}{$\gAND$}}}}\,\,$}}}

\def\gdra{\mbox{$\,{\rotatebox[origin=c]{-90}{\raisebox{0.12ex}{$\gand$}}{\mkern-1mu\raisebox{0.065ex}{\aol}}}\,$}}
\def\gDRA{\mbox{$\,\,{\rotatebox[origin=c]{-90}{\raisebox{0.12ex}{$\gAND$}}{\mkern-0.8mu\AOL}}\,\,$}}

\def\gdla{\mbox{\rotatebox[origin=c]{-180}{$\,{\rotatebox[origin=c]{-90}{\raisebox{0.12ex}{$\gand$}}{\mkern-1mu\raisebox{0.065ex}{\aol}}}\,$}}}
\def\gDLA{\mbox{\rotatebox[origin=c]{180}{$\,\,{\rotatebox[origin=c]{-90}{\raisebox{0.12ex}{$\gAND$}}{\mkern-0.8mu\AOL}}\,\,$}}}
%

%
%
\newcommand{\pgsc}{\mbox{$\,;\,$}}
\newcommand{\pgSC}{\mbox{\Large{\boldmath{$\,;\,$}}}}
\newcommand{\AST}{\mbox{\boldmath{$\ast$}}}
\newcommand{\pgand}{\mbox{$\cap$}}
\newcommand{\pgAND}{\mbox{$\,\bigcap\,$}}
\newcommand{\pgor}{\mbox{$\cup$}}
\newcommand{\pgOR}{\mbox{$\,\bigcup\,$}}
\def\pgra{\mbox{$\,{\aol{\mkern-1.5mu{\rotatebox[origin=c]{-90}{\raisebox{0.12ex}{$\aand$}}}}}\,$}}
\def\pgRA{\mbox{$\,\,{\AOL{\mkern-0.4mu{\rotatebox[origin=c]{-90}{\raisebox{0.12ex}{$\aAND$}}}}}\,\,$}}
\def\pgla{\mbox{$\,{\rotatebox[origin=c]{180}{$\aol{\mkern-1.6mu{\rotatebox[origin=c]{-90}{\raisebox{0.12ex}{$\aand$}}}}$}}\,$}}
\def\pgLA{\mbox{\rotatebox[origin=c]{180}{$\,\,{\AOL{\mkern-0.4mu{\rotatebox[origin=c]{-90}{\raisebox{0.12ex}{$\aAND$}}}}}\,\,$}}}
\def\pgdra{\mbox{$\,{\rotatebox[origin=c]{-90}{\raisebox{0.12ex}{$\aand$}}{\mkern-1.2mu\aol}}\,$}}
\def\pgDRA{\mbox{$\,\,{\rotatebox[origin=c]{-90}{\raisebox{0.12ex}{$\aAND$}}{\mkern-0.8mu\AOL}}\,\,$}}
\def\pgdla{\mbox{\rotatebox[origin=c]{180}{$\,{\rotatebox[origin=c]{-90}{\raisebox{0.12ex}{$\aand$}}{\mkern-1.2mu\aol}}\,$}}}
\def\pgDLA{\mbox{\rotatebox[origin=c]{180}{$\,\,{\rotatebox[origin=c]{-90}{\raisebox{0.12ex}{$\aAND$}}{\mkern-0.8mu\AOL}}\,\,$}}}
%

%
\newcommand{\agand}{\mbox{$\sqcap$}}
\newcommand{\agAND}{\mbox{$\,\,\bigsqcap\,\,$}}
\newcommand{\agor}{\mbox{$\sqcup$}}
\newcommand{\agOR}{\mbox{$\,\,\bigsqcup\,\,$}}
\def\agra{\mbox{$\,{\aol{\mkern-1.5mu{\rotatebox[origin=c]{-90}{\raisebox{0.12ex}{$\agand$}}}}}\,$}}
\def\agRA{\mbox{$\,\,{\AOL{\mkern-0.4mu{\rotatebox[origin=c]{-90}{\raisebox{0.12ex}{$\agAND$}}}}}\,\,$}}
\def\agla{\mbox{$\,{\rotatebox[origin=c]{180}{$\aol{\mkern-1.6mu{\rotatebox[origin=c]{-90}{\raisebox{0.12ex}{$\agand$}}}}$}}\,$}}
\def\agLA{\mbox{\rotatebox[origin=c]{180}{$\,\,{\AOL{\mkern-0.4mu{\rotatebox[origin=c]{-90}{\raisebox{0.12ex}{$\agAND$}}}}}\,\,$}}}
\def\agdra{\mbox{$\,{\rotatebox[origin=c]{-90}{\raisebox{0.12ex}{$\agand$}}{\mkern-1.1mu\aol}}\,$}}
\def\agDRA{\mbox{$\,\,{\rotatebox[origin=c]{-90}{\raisebox{0.12ex}{$\agAND$}}{\mkern-0.7mu\AOL}}\,\,$}}
\def\agdla{\mbox{\rotatebox[origin=c]{180}{$\,{\rotatebox[origin=c]{-90}{\raisebox{0.12ex}{$\agand$}}{\mkern-1.1mu\aol}}\,$}}}
\def\agDLA{\mbox{\rotatebox[origin=c]{180}{$\,\,{\rotatebox[origin=c]{-90}{\raisebox{0.12ex}{$\agAND$}}{\mkern-0.7mu\AOL}}\,\,$}}}
\def\agTOP{{_{\textrm{g\!\!}}\textrm{T}}}
\def\agtop{{_{\textrm{g\!\!\!}}\top}}
\def\agBOT{\textrm{\rotatebox[origin=c]{180}{T}}^{\textrm{\!g}}}
\def\agbot{\bot^{\textrm{\!\!g}}}
\def\agneg{\mbox{$\mkern-0.4mu\sim\mkern-0.4mu$}}

\def\aga{\texttt{a}}
\def\agb{\texttt{b}}
\def\agc{\texttt{c}}
\def\agd{\texttt{d}}
\def\agA{\Large{\texttt{a}}}
\def\agB{\Large{\texttt{b}}}
\def\agC{\texttt{C}}
\def\agD{\texttt{D}}

\def\bulletaga{{\bullet_{\!\aga}}}
\def\circaga{{\circ_{\!\aga}}}
\def\Wdiaaga{{\Wdia_{\!\aga}}}
\def\Bdiaaga{{\Bdia_{\!\aga}}}
\def\Wboxaga{{\Wbox_{\aga}}}
\def\Bboxaga{{\Bbox_{\aga}}}
\def\andol{\rule[-0.4563ex]{1.38ex}{0.1ex}}
\def\aol{\rule[0.5865ex]{1.38ex}{0.1ex}}
\def\AOL{\rule[0.65ex]{1.45ex}{0.1ex}}
\def\orol{\rule[1.4253ex]{1.38ex}{0.1ex}}
\newcommand{\mNEG}{\mbox{$\circledast$}}
\newcommand{\mneg}{\mbox{$\ast$}}
\def\mandi{\mbox{$\mkern+2mu\vartriangle_{1}\mkern+2mu$}}
\def\mANDi{\mbox{$\,{{\bigwedge{\mkern-16.2mu\andol}_{1}}}\,\,$}}
\def\mandii{\mbox{$\mkern+2mu\vartriangle_{2}\mkern+2mu$}}
\def\mANDii{\mbox{$\,{{\bigwedge{\mkern-16.2mu\andol}_{2}}}\,\,$}}

\def\mTOP{\textrm{T}}
\def\mtop{\top}
\def\mBOT{\textrm{\rotatebox[origin=c]{180}{T}}}
\def\mbot{\bot}
\def\mneg{\rightharpoondown}


\newcommand{\RESalphaProxy}{\,\rotatebox[origin=c]{90}{$\{\rotatebox[origin=c]{-90}{$\alpha$}\}$}\,}
\newcommand{\RESBoxalphaProxy}{\,\rotatebox[origin=c]{90}{$[\rotatebox[origin=c]{-90}{$\alpha$}\}$}\,}
\newcommand{\RESDiaalphaProxy}{\,\rotatebox[origin=c]{90}{$\{\rotatebox[origin=c]{-90}{$\alpha$}\rc$}\,}

\newcommand{\RESalphajProxy}{\,\rotatebox[origin=c]{90}{$\{\rotatebox[origin=c]{-90}{$\!\alpha_{\!j}$}\,\}$}\,}

\newcommand{\RESbetaProxy}{\,\rotatebox[origin=c]{90}{$\{\rotatebox[origin=c]{-90}{$\beta$}\}$}\,}

\newcommand{\RESepsilonProxy}{\,\rotatebox[origin=c]{90}{$\{\rotatebox[origin=c]{-90}{$\varepsilon$}\}$}\,}

\newcommand{\RESalphabetaProxy}{\,\rotatebox[origin=c]{90}{$\Big\{\rotatebox[origin=c]{-90}{$\alpha\cdot\beta$}\Big\}$}\,}
\newcommand{\RESbetaalphainvProxy}{\,\rotatebox[origin=c]{90}{$\Bigg\{\rotatebox[origin=c]{-90}{$\beta\cdot\alpha^{-1}$}\Bigg\}$}\,}
\newcommand{\RESalphacupbetaProxy}{\,\rotatebox[origin=c]{90}{$\Bigg\{\rotatebox[origin=c]{-90}{$\alpha\cup\beta$}\Bigg\}$}\,}

\newcommand{\RESalphaBox}{\,\rotatebox[origin=c]{90}{$[\rotatebox[origin=c]{-90}{$\alpha$}]$}\,}
\newcommand{\RESalphaplusBox}{\,\rotatebox[origin=c]{90}{$\Big[\rotatebox[origin=c]{-90}{$\,\,\alpha^+$}\Big]$}\,}

\newcommand{\RESalphajBox}{\,\rotatebox[origin=c]{90}{$[\rotatebox[origin=c]{-90}{$\!\alpha_{\!j}$}\,]$}\,}

\newcommand{\RESalphaDia}{\,\rotatebox[origin=c]{90}{$\langle\rotatebox[origin=c]{-90}{$\alpha$}\rangle$}\,}
\newcommand{\RESalphaplusDia}{\,\rotatebox[origin=c]{90}{$\Big\langle\rotatebox[origin=c]{-90}{$\,\,\alpha^+$}\Big\rangle$}\,}

\newcommand{\RESalphajDia}{\,\rotatebox[origin=c]{90}{$\langle\rotatebox[origin=c]{-90}{$\!\alpha_{\!j}$}\,\rangle$}\,}

\newcommand{\Bigsemic}{\mbox{\Large {\bf ;}}}

\makeatletter
\newcommand{\WKnowProxy}[2]{%
  {\mathbin{\ooalign{$#1\circ#2 $\cr\hidewidth
   \raise.155ex\hbox{$#1{\scriptstyle{\ast}}#2$}\hidewidth\cr  }}}}
\makeatother

\makeatletter
\newcommand{\BKnowProxy}[2]{%
  {\mathbin{\ooalign{$#1\bullet#2 $\cr\hidewidth
   \raise.155ex\hbox{$#1{\scriptstyle{\color{white}{\ast}}}#2$}\hidewidth\cr  }}}}
\makeatother

\makeatletter
\newcommand{\WKnow}[2]{%
  {\mathbin{\ooalign{$#1\Wbox#2 $\cr\hidewidth
   {$#1\ast#2$}\hidewidth\cr  }}}}
\makeatother

\makeatletter
\newcommand{\BKnow}[2]{%
  {\mathbin{\ooalign{$#1\Bbox#2$\cr\hidewidth
   {$#1\color{white}{\ast}#2$}\hidewidth\cr  }}}}
\makeatother

\makeatletter
\newcommand{\WDualKnow}[2]{%
  {\mathbin{\ooalign{$#1\Wdia#2 $\cr\hidewidth
   {$#1\ast#2$}\hidewidth\cr  }}}}
\makeatother

\makeatletter
\newcommand{\BDualKnow}[2]{%
  {\mathbin{\ooalign{$#1\Bdia#2 $\cr\hidewidth
   {$#1\color{white}{\ast}#2$}\hidewidth\cr  }}}}
\makeatother

\newcommand{\PAIRalpha}{\{\!\!\!\,\RESalphaProxy\,\!\!\!\}}

\newcommand{\fns}{\footnotesize}
\newcommand{\mr}{\multirow}
\newcommand{\rest}{\upharpoonright}

\begin{abstract}
The framework
developed in the present paper provides a formal ground to generate and study explainable categorizations of sets of entities, based on 
the epistemic attitudes of individual agents or groups thereof. 
Based on this framework, we  discuss a machine-leaning  meta-algorithm for  outlier detection and classification which  provides local and global explanations of its results.

\end{abstract}
\section{Introduction}
\setlength{\abovedisplayskip}{0.2mm}
\setlength{\belowdisplayskip}{0.2mm}
\setlength{\abovedisplayshortskip}{0.2mm}
\setlength{\belowdisplayshortskip}{0.2mm}

Categories are cognitive tools used both by humans and machines to organize experience, understand and function in their environment, and understand and interact with each
other, by grouping things together which can be meaningfully compared and evaluated. Categorization is the basic operation humans perform e.g.~when they relate experiences/actions/objects in the present to those in the past, thereby recognizing them
as instances of the same type and being able to compare them with each other. This is what humans do when trying to understand what
an object is or does, or what a situation means, and when making judgments or decisions based on experience. Categorization is the single cognitive mechanism underlying {\em meaning}-attribution, {\em value}-attribution and {\em decision-making}, and therefore it is the background environment in which these three cognitive processes
can formally be analyzed in their relationships to one another. Nowadays, categories are key to the theories and methodologies of several research areas at the interface of social sciences and AI, not only because categories underlie a wide array of phenomena  key to these areas, spanning from the mechanisms of perception \cite{ROY2005170}, to the creation of languages \cite{ALOMARI2022103637}, social identities and cultures \cite{PEREIRA20161}, but also because category-formation is core to the development of data-analytic  techniques based on AI, such as data mining \cite{nothman2013learning,kuznetsov2013knowledge},  text analysis \cite{cimiano2005learning}, social network analysis \cite{freeman1993using,hao2015k},  which have become fundamental tools in empirical sciences. 

\paragraph{Motivation and main contribution.}
The present paper starts from the rather self-evident observations that one and the same set of entities can be categorized in very different ways, and that certain ways to categorize might be more suitable or appropriate than others for accomplishing a certain task. In the present paper, we  introduce  a formal framework specifically designed to map the space of categorization systems (over a given set of entities) in a parametric way which systematically links these systems with key aspects of agency (e.g.~goals, but we  expand on this point below), and helps to identify which categorization systems are optimal {\em relative} to
 specific tasks,  such as {\em outlier detection} in the context of financial auditing, which will be our running example. 


For instance, the experienced auditors’ evaluation of evidence is rooted in a process of categorization of the pieces of evidence (e.g.~financial transactions)  which is possibly very different from the ``official'' categorization system  through which the evidence is presented in the self-reported financial statement of the given firm. These categories provide the context of evaluation in which different 
pieces of evidence are compared with/against each other. 

\paragraph {Interrogative agendas.} The final outcome of the auditing process is the formation of a qualitative opinion, by expert auditors, on the fairness and completeness of a given firm's financial accounts. Towards the formation of their opinion, the auditors might  not attribute the same  importance to all the  features of the pieces of evidence, and they might also disagree with regard to the relative importance of certain features. 
In the present  framework, the different epistemic attitudes of agents are captured by the notion of {\em interrogative agenda} 
\cite{enqvist2012modelling}.
This notion can be understood as the `questions' that an agent sets to resolve (by gathering information) before making a decision. Interestingly, the level of {\em expertise} of an agent can be captured by how good the questions they ask are. In the present framework,  a simple way for representing an agent's interrogative agenda  is as a designated set of features, i.e.~the set of features that the agent considers significantly more relevant than others.  
However, in many cases,  an agent's agenda  may not be  realistically approximated in such a simple way, but might consist of different relevance or importance values assigned to  different sets of features. Such agendas will be represented by Dempster-Shafer mass functions. Independently of their mathematical representation,  interrogative agendas will induce categorization systems on a given set of entities, 
and hence parametrize the space of categorization systems. 

\paragraph{Background theory.} The present framework is set within {\em  Formal Concept Analysis} (FCA) \cite{wille1996formal, belohlavek1999fuzzy}, and is based on   {\em formal contexts} \cite{wille1996formal}, i.e.~structures consisting of  domains $A$ (of objects) and $X$ (of features), and binary relations $I$ between them. The running example mentioned above concerns the formal context arising from a network (bipartite graph) of financial transactions (i.e.~a {\em financial statements network}) \cite{boersma2018financial}.

Mathematically,  bipartite graphs   and formal contexts are 
isomorphic structures; moreover, both types of structures are used to represent databases \cite{wille1996formal, vskopljanac2014formal, kuznetsov2004machine, pavlopoulos2018bipartite, hayes2004bipartite, ravasz2002hierarchical, YIN2019105020, cobb2003application}. Setting the present framework on  formal contexts allows for access to a mathematically principled way to generate categorization systems, in the form of the  construction of the lattice of {\em formal concepts} of a formal context \cite{birkhoff1940lattice}, as well as a suitable  base for expanding the present formal framework to represent and support vagueness, epistemic uncertainty, evidential reasoning, and incomplete information \cite{WU200938, conradie2017toward, conradie2021rough, frittella2020toward}.   
In particular, representing categorization systems as lattices allows for
   hierarchical, rather than flat, categorizations of objects, as well as
    for   a more structured control of the categorization, based on the generation of categories from arbitrary subsets of objects or features.  

\paragraph{Meta-algorithm.} Based on the present framework, we discuss  an explainable meta-algorithm for outlier detection or classification,  introduced in \cite{acar2023meta},  in which  the categorization systems are not pre-specified but are `learned', based on the accuracy of their  prediction. This algorithm is better understood in the light of the  framework presented in this paper.

\paragraph{Related work.} Extant interdisciplinary research in cognition, psychology, and management science has drawn attention to the issue of studying the implications of different ways to categorize \cite{smith1998alternative, lewandowsky2000competing, cohen1987alternative}. 
The approach adopted in the present paper, for parametrizing the space of possible categorizations in terms of agendas, is similar to  feature selection (FS) for categorization \cite{dash1997feature,li2017feature}. However, the present paper departs  from this literature since it emphasises  the connection between the selection of features and the epistemic attitudes  of the agents.  Several FS algorithms exist, making it difficult to choose the best FS strategy for a specific task \cite{bolon2019ensembles,opitz1999feature,ben2018ensemble}.   One possible solution for this problem is to use ensembles of FS algorithms \cite{li2017feature,tang2019feature,xue2021multi,tang2014feature}. 
The present paper uses a similar `ensembling' strategy in the proposed method, discussed in Section  \ref{ssec:stability-based method}, for associating  a categorization system with a given non-crisp agendas, and a meta-learning algorithm for classification and outlier detection (cf.~Section \ref{sec:algorithms}).
In the literature in explainable AI, it has been argued  that  FS  can make algorithms more interpretable by reducing the dimensionality of inputs \cite{lazebnik2024algorithm,linardatos2020explainable,rosenfeld2021better}. In the same way, Algorithm \ref{algo:metaxai} can be used to improve the explainability of  algorithms for classification and outlier detection. An instantiation  of this meta-algorithm is used in \cite{boersma2023outlier} as  an explainable outlier detection  algorithm. 


\paragraph{ Structure of the paper.} In Section \ref{sec:Prelim}, we collect the basic definitions and facts pertaining to FCA and Dempster-Shafer theory. In Section \ref{sec:Interrogative agendas, coalitions, and categorization}, we introduce the (crisp and non-crisp) formal  framework for generating a set of categorization systems, over a given set of entities, parametrized by interrogative agendas, and introduce ways of aggregating (crisp and non-crisp) agendas. In Section \ref{sec:Example}, we illustrate this framework by means of our running example on financial statements network.  

In Section \ref{ssec:stability-based method}, we introduce a methodology for associating single categorization systems to non-crisp agendas (i.e.~Dempster-Shafer mass functions), and define a partial order on non-crisp agendas with respect to which this mapping  is order-preserving. In Section \ref{sec: orderings}, we compare this partial order to other orderings of mass functions in the literature, and study its properties with respect to the aggregated  agendas.
In Section \ref{sec:algorithms}, we discuss a meta-learning algorithm for outlier detection and classification. 
We conclude and mention some directions for future research in Section \ref{sec:Conclusion and further directions}.

\section{Preliminaries}\label{sec:Prelim}
\paragraph{Formal contexts and their concept lattices}
A {\em formal context} \cite{ganter2012formal}  is a structure $\mathbb{P} = (A, X, I)$ such that $A$ and $X$ are sets, and $I\subseteq A\times X$ is a binary relation. 
Formal contexts can be thought of as abstract representations of databases, where elements of $A$ and $X$ represent objects and features, respectively, and the relation $I$ records whether a given object has a given feature. 
Every formal context as above induces maps $I^{(1)}: \mathcal{P}(A)\to \mathcal{P}(X)$ and $I^{(0)}: \mathcal{P}(X)\to \mathcal{P}(A)$, respectively defined by the assignments 
\smallskip

{{\centering

     $I^{(1)}[B]\coloneqq\{x\in X\mid \forall a(a\in B\Rightarrow aIx)\}$ \quad  and \quad 
 $I^{(0)}[Y] \coloneqq \{a\in A\mid \forall x(x\in Y\Rightarrow aIx)\}$.

\par}}
\smallskip
 
A {\em formal concept} of $\mathbb{P}$ is a pair 
$c = (\val{c}, \descr{c})$ such that $\val{c}\subseteq A$, $\descr{c}\subseteq X$, and $I^{(1)}[\val{c}] = \descr{c}$ and $I^{(0)}\descr{c} = \val{c}$. 
A subset $B \subseteq A$ (resp.\ $Y\subseteq X$) is said to be {\em closed}, or {\em Galois-stable}, if $\mathsf{Cl}_1(B)=I^{(0)}[I^{(1)}[B]]=B$ (resp.\ $\mathsf{Cl}_2(Y)=I^{(1)}[I^{(0)}[Y]]=Y$).
The set of objects $\val{c}$ is  the {\em extension} of the concept $c$, while  the set of features $ \descr{c}$ is  its {\em intension}\footnote{The symbols $\val{c}$ and $\descr{c}$, respectively denoting the extension and the intension of a concept $c$, have been introduced and used in the context of a research line aimed at developing the logical foundations of categorization theory, by regarding formulas as names of categories (formal concepts), and interpreting them as formal concepts arising from  given formal contexts  \cite{conradie2017toward,conradie2021rough,frittella2020toward,conradie2016categories,conradie2019logic}.}. 
The set ${\mathrm{L}}(\mathbb{P})$  of the formal concepts of $\mathbb{P}$ can be partially ordered as follows: for any $c, d\in {\mathrm{L}}(\mathbb{P})$, 
\begin{equation}
c\leq d\quad \mbox{ iff }\quad \val{c}\subseteq \val{d} \quad \mbox{ iff }\quad \descr{d}\subseteq \descr{c}.
\end{equation}
With this order, ${\mathrm{L}}(\mathbb{P})$ is a complete lattice, the {\em concept lattice} $\mathbb{P}^+$ of $\mathbb{P}$. As is well known, any complete lattice $\mathbb{L}$ is isomorphic to the concept lattice $\mathbb{P}^+$ of some formal context $\mathbb{P}$ \cite{birkhoff1940lattice}. Throughout this paper, we will often identify the concept lattice $\mathbb{P}^+$ with the lattice of the extensions of its formal concepts, ordered by inclusion, and will write e.g.~$G\in \mathbb{P}^+$ for a subset $G\subseteq A$ to signify that $G$ is the extension of some formal concept of $\mathbb{P}$.

\paragraph{Discretization of continuous attributes and conceptual scaling.} \label{par:Conceptual scaling}
{The framework discussed above can also be applied to cases in which the relation $I\subseteq A\times X$ can take continuous values, modulo a process   known as {\em conceptual scaling} \cite{ganter1989conceptual}. Scaling is an important part of most FCA-based techniques and has been studied extensively \cite{ganter1989conceptual,prediger1997logical,prediger1999lattice}. Choosing the correct scaling method depends on the specific task the concept lattice is used for.}

\paragraph{Belief, plausibility and mass functions} Here we recall standard notation and terminology from Dempster-Shafer theory \cite{sentz2002combination,yager2008classic}.
\label{def:bel-func}
A {\em belief function} (cf.~\cite[Chapter 1]{shafer1976mathematical}) on a set $S$ is a map $\bel: \mathcal{P}(S)\to [0,1]$ such that 
$\bel(S)=1$,  and for every $n\in \mathbb{N}$,
\[
\bel (A_1 \cup, \cdots, \cup A_n)  \geq  
\sum_{\varnothing \neq I \subseteq \{1, \cdots, n\}}
(-1)^{|I|+1} \bel \left (\bigcap_{i \in I} A_i \right).
\]
A {\em plausibility function on} $S$ is a map $\pl: \mathcal{P}(S)\to [0,1]$ such that 
$\pl(S)=1$,  and for every $n\in \mathbb{N}$,
\[
\pl (A_1 \cup A_2 \cup, \cdots, \cup A_n)  \leq 
\sum_{\varnothing \neq I \subseteq \{1,2, \cdots ,n\} }
(-1)^{|I| +1}\pl 
\left( \bigcap_{i \in I} A_i 
\right).
\]
For any set $X$, let $\overline{X}$ be its complete $S \setminus X$. Belief and plausibility functions on sets are interchangeable notions: 
for every belief function $\bel$ as above, the assignment  $X\mapsto 1- \bel(\overline{X})$ defines a plausibility function on $S$, and for every plausibility function $\pl$ as above, the assignment  $X\mapsto 1- \pl(\overline{X})$ defines a belief function on $S$. Let $S$ be any set.

A {\em (Dempster-Shafer) mass function} is a map $\mass: \mathcal{P}(S)\to [0,1]$ such that 
$\sum_{X \subseteq S} \mass (X) = 1$.

On finite sets, belief (resp.~plausibility) functions and mass functions are interchangeable notions:  any mass function $\mass$ as above induces the belief function   $\bel_{\mass}: \mathcal{P}(S)\to [0,1]$ defined as 
\begin{equation} 
\bel_\mass(X) := \sum_{Y \subseteq X} \mass(Y) \qquad \text{ for every } X \subseteq S,
\end{equation}
and  a plausibility function 
\begin{equation} 
\pl_\mass(X) := \sum_{Y \cap X \neq \emptyset} \mass(Y) \qquad \text{ for every } X \subseteq S.
\end{equation}
Conversely, any belief function $\bel$ as above induces the mass function   $\mass_{\bel}: \mathcal{P}(S)\to [0,1]$ defined as  
\begin{equation} 
\mass_\bel (X) := \bel(X) - \sum_{Y \subseteq X} (-1)^{|X \smallsetminus Y|} \bel(Y) \quad \text{ for all } X \subseteq S.
\end{equation}
For any mass function $m:\mathcal{P}(X) \to [0,1]$, its associated {\em quality function} ${q}_m$ is  
\[
{q}_m(Y) =\sum_{Y \subseteq Z}m(Z)
\quad\quad\quad \text{ for all } Y \subseteq X.
\]

\section{Categorizations induced by interrogative agendas} \label{sec:Interrogative agendas, coalitions, and categorization} 
In this section, we introduce two  multi-agent frameworks, each of which represents the interrogative agendas associated with {\em individual} agents and {\em groups} of agents, as well as  the categorizations induced by these agendas. 
These two frameworks differ in the way interrogative agendas are represented; namely as {\em subsets} of features, and as {\em Dempster-Shafer mass functions} over the set of features, respectively. We  refer to the interrogative agendas represented in the former way as  the {\em crisp} ones, and to those represented in the latter way as the {\em non-crisp} ones.\footnote{While an interrogative agenda modelled as a subset gives rise to a single categorization system, when modelled as a Dempster Shafer mass function, it gives rise to  different priority or preference values assigned to different sets of features, which in turn induces a mass function on a whole spectrum of  categorization systems.}

\paragraph{Crisp framework.} We consider tuples $(\mathbb{P}, C, R)$
 such that $\mathbb{P} = (A,X,I) $ is a  finite (many-valued) formal context, $C$ is a finite set, and $R\subseteq X\times C$ is a binary relation.  The formal context $\mathbb{P}$ captures, as usual, the set $A$ of the objects to be categorized, and the set $X$ of features attributed to each object through the incidence relation $I$;   the set $C$ is understood as a set of agents, and the relation $R$  associates any agent $j\in C$ with their {\em (crisp) interrogative agenda}, represented as the set $X_j: = R^{-1}[j] = \{x\in X\mid xRj\}$ of features $x\in X$ which $j$ considers relevant. 

 Interrogative agendas can also be associated with {\em coalitions} of agents  as follows: for any  $c\subseteq C$, we let \[\Diamond c: = \bigcap \{X_j\mid j\in c\}\quad\quad \text{ and } {\rhd} c: = \bigcup \{X_j\mid j\in c\}\]
 denote the {\em common} and the {\em distributed} interrogative agendas associated with $c$, respectively.

 For any $Y\subseteq X$, let $\mathbb{P}_{Y}: = (A, Y, I_{Y})$, where $I_{Y}: = I\cap (A\times Y)$. Hence, if $Y$ is an interrogative agenda, $Y$ induces a categorization system on the objects of $A$, which
  is represented by the concept lattice $\mathbb{P}_{Y}^+$.

  Notice that if $Z\subseteq Y\subseteq X$, and $B\in \mathbb{P}_Z$ (i.e.~$B = \{a\in A\mid \forall x(x\in Z'\Rightarrow xI_Z a)\}$ for some $Z'\subseteq Z$) then $B\in \mathbb{P}_Y$; indeed, to show that some $Y'\subseteq Y$ exists such that $B = \{a\in A\mid \forall x(x\in Y'\Rightarrow xI_Y a)\}$, it is enough to let $Y': = Z'$. This shows that, for every $B\in\mathbb{P}^+$, the set $\mathcal{V}_B: = \{Y\subseteq X\mid B\in\mathbb{P}^+_Y\}$ is upward closed w.r.t.~inclusion. Consequently, for any coalition $c$, and  any agent $j$ in coalition $c$, 
  \begin{equation}\label{eq:Diamond-triangle}
      \mathbb{P}^+_{\Diamond c} \subseteq  \mathbb{P}^+_{X_j} \subseteq  \mathbb{P}^+_{\rhd c}
  \end{equation}
  \smallskip
  The parametric structure of $\{\mathbb{P}_Y^+\mid Y\subseteq X\}$ allows for the possibility  to select those categorization systems with `meaningful categories' relative to the task at hand: larger agendas induce finer categorization systems, capable of making more distinctions among objects, and smaller agendas induce coarser categorization systems, potentially suitable for e.g.~identifying outliers while reducing the number of false positives.

\paragraph{Non-crisp framework. } We consider tuples $(\mathbb{P}, C, \mathcal{M})$ such that $\mathbb{P}$ and $C$ are as above, and $\mathcal{M} = \{m_j: \mathcal{P}(X)\to [0. 1]\mid j\in C\}$ is a $C$-indexed set of Dempster-Shafer mass functions.
Intuitively, if  $j\in C$ and  $Y\subseteq X$, then the value $m_j(Y)$ represents agent $j$'s preference   to use the concept lattice $\mathbb{P}_Y^+$ associated with $Y$ as  a  categorization system.  The following definition captures  particularly interesting cases: 
 
\begin{definition}
For any Dempster-Shafer mass function  $m:\mathcal{P}(S) \to [0,1]$, a set $Y \subseteq S$ is a {\em focal set} of $m$ if $m(Y)>0$. If $m$ has at most one focal set $X \subset S$, then 
 $m$ is {\em simple}.
 \end{definition}
For instance, if we only have information about agent $j$'s opinion on the relevance of a certain  subset  $Y \subseteq X$  (quantified as  $\alpha \in [0,1]$), this information can be encoded by representing the agenda of  $j$  as the simple mass function $m_j$ such that $m_j(Y)=\alpha$ and $m_j(X)=1-\alpha$. Assigning the remaining mass to $X$ is  motivated by the idea that if an agent is unsure about which categorization is preferable for a given task, then the agent will prefer the finest possible categorization i.e.~the categorization generated by  all available features, so as to not ignore any features that may be relevant to the given task.
 For example, the FCA-based outlier detection algorithm used in \cite{boersma2023outlier} is more likely to flag an object  as an outlier in a finer categorization than in a coarser categorization. However, if an agent believes that a coarser categorization (i.e.~one associated with a smaller agenda) is already capable of flagging outliers, then the agent will prefer this one, for the sake of avoiding false positives. However, if an agent is unsure whether a smaller agenda suffices to flag outliers, that agent will prefer the finest categorization system available so to avoid false negatives.


If agent $j$ has opinions on the relevance of  each individual feature, the agenda of $j$ can be represented   by a  mass function $m_j: \mathcal{P}(X) \to  [0,1]$, for any $Y \subseteq X$, 
\[
m_j(Y) =\sum_{y \in Y}v_j(y),
\]
where $v_j(y)$ is the  (normalized) importance value assigned by  $j$ to each $y\in X$. 

Each $m_j\in \mathcal{M}$ induces a probability mass function $m_j': \mathcal{R} \to [0,1]$, where $\mathcal{R}: = \{\mathbb{P}_Y^+\mid Y\subseteq X\}$,  defined by the assignment $m_j'(\mathbb{P}_Y^+)= m_j(Y)$.
Intuitively, for any $Y\subseteq X$, the value $m_j'(\mathbb{P}_Y^+)$ represents the extent to which agent $j$ prefers  the categorization system induced by $Y$. Moreover, each $m_j$ induces a probability function $p_{j} : \mathcal{P}(\mathcal{P}(X))\to [0, 1]$ defined by the assignment $p_{j}(\mathcal{V}) =\sum_{ Y \in \mathcal{V}}m_j(Y)$
 for any $\mathcal{V} \subseteq \mathcal{P}(X)$.

  For any coalition $c$, we can associate (non-crisp) interrogative agendas $\oplus c$, $\Diamond c$, and $\rhd c$ with $c$ as follows:

\begin{enumerate}
    \item  For any $Y \subseteq X$, $Y\neq \varnothing$  
    \begin{equation}\label{eq:DS-rule}
     (\oplus c) ({Y}) = \frac{\sum \{\prod_{j \in c} m_j({Z}_j)\mid   \bigcap_{j \in c} Z_j=Y\}} {\sum \{\prod_{j \in c} m_j({Z}_j)\mid \bigcap_{j \in c} Z_j \neq \varnothing\}}
 \end{equation}
 and $ (\oplus c) (\varnothing) =0$. This aggregation  is the  Dempster-Shafer combination   \cite{shafer1976mathematical} of the mass functions $m_j$  for $j \in c$. The normalization in the above rule 
 allows us to ignore completely contradictory agendas (i.e.~agendas with empty intersection), thus giving more weight to issues which have consensus of the agents. However, in some scenarios, we may want to allow mass on the empty set of features (which corresponds to a categorization in which all elements are in same category). 
 The value $m(\varnothing$) describes the agent's preference for categorization with only one class. For such scenarios, rather than taking $\oplus c$ as the aggregated agenda, we may use the unnormalized mass function $\Diamond c$  defined as follows:
 
 \begin{equation}
   (\Diamond c)(Y) = \sum \{\prod_{j \in c} m_j({Z}_j)\mid \bigcap_{j \in c} Z_j=Y\}  
 \end{equation}

     \item  For any ${Y} \subseteq X$, 
 \begin{equation}\label{eq:inverse DS-rule} 
   ({\rhd} c)(Y) = \sum \{\prod_{j \in c} m_j({Z}_j)\mid  \bigcup_{j \in c} Z_j=Y\}
 \end{equation}
\end{enumerate}
\begin{remark}
Crisp agendas can be regarded as non-crisp agendas with a single focal element $Z_j$ for every agent $j$. In this case, for every coalition $c$,
the mass functions $\Diamond c$ and $\rhd c$ also have  single focal elements $\bigcap_{j \in c} Z_j$ and   $\bigcup_{j \in c} Z_j$, respectively, which coincide with the crisp agendas $\Diamond c$ and $\rhd c$. 
This justifies our use of the same symbols for these operations in the crisp and non-crisp setting.  
\end{remark}

 Several rules have been used in Dempster-Shafer theory  for aggregating preferences of different agents \cite{sentz2002combination,frittella2020toward}, and their applicability to the present framework is still widely unexplored.  A concrete scenario involving non-crisp agendas is discussed in the next section.


\section{Example: financial statements network}\label{sec:Example}
In this section, we illustrate the ideas  discussed above  by way of an example in the context of financial statements networks.
These  are a type of data structure for organizing the information contained in the journal entries of a company, which record  the transfers from one set of financial accounts to another set. These entries are generated by their underlying 
%
{\em business processes}, which can be formally defined as follows \cite{boersma2018financial,boersma2020reducing}:
\begin{equation} \label{eq:business process definition}
    a: \sum_{1\leq i\leq m}\alpha_ix_i \implies \sum_{1\leq j\leq n} \beta_jy_j
\end{equation}
where $m$ is the number of credited financial accounts, $n$ is the number of debited financial accounts, $\alpha_i$ is the relative amount with respect to the total credited, and $\beta_j$  the relative amount with respect to the total debited. The arrow  represents the flow of money between the accounts.

A {\em financial statements network} is a bipartite digraph $\mathbb{G}=(A \cup X, E)$  in which $A$ is the set of business processes,  $X$ is the set  of financial accounts, and the (many-valued) directed edges in $E \subseteq X \times A$ record information on  (the share of) a given financial
account in a given business process. Clearly, each such $\mathbb{G}=(A \cup X, E)$  can  equivalently be represented as a many-valued formal context $\mathbb{P} = (A, X, I)$ where $I$ is the converse of $E$. 
For each business process $a\in A$, the (weighted) edges between nodes $x_i$ (resp.~$y_j$) and  $a$ are the coefficients $\alpha_i$ (resp.~$\beta_j$) in \eqref{eq:business process definition}.

Consider the financial statements network 
presented in Table \ref{database table} of  Appendix \ref{sec:dataset}, with  business processes $A: = \{ a_1, a_2, \ldots, a_{12}\}$ and   financial accounts $X := \{x_1, x_2, \ldots, x_6\}$ specified as follows:

\begin{center}
    \begin{tabular}{|c|c|c|c|c|c|}
    \hline
      $x_1$ & tax & $x_2$ & revenue &
       $x_3$&cost of sales \\ \hline
       $x_4$& personnel expenses &
      $x_5$& inventory & $x_6$ &other expenses\\
      \hline
    \end{tabular}
\end{center}
Table \ref{table: context table} encodes the many-valued formal context $\mathbb{P} = (A, X, I)$ extracted from this database. Each cell of  Table \ref{table: context table}  reports the value of the  relation $I: A\times X\to [-1, 1]$,  which, for any  process $a$ and account $x$, represents the share of $x$ in $a$. 

Using interval scaling\footnote{Interval scaling is one of the methods used commonly for conceptual scaling. For more, see \cite{ganter1989conceptual}.}, we  convert the many-valued formal context $\mathbb{P} = (A, X, I)$  into the 2-valued formal context $\mathbb{P}^{(s)} = (A, X^{(s)}, I^{(s)})$, for $s\in \mathbb{N}$,   where 
$X^{(s)}:= \{x_{ik}\mid 1 \leq i \leq 6, 1 \leq k \leq s\}$, and $I^{(s)}\subseteq A\times X^{(s)}$ is such that $a I^{(s)} x_{ik}$  iff $ I(a,x_{i})\in \left[-1 +\frac{2(k-1)}{s},   -1+ \frac{2k}{s} \right]$.

The concept lattice corresponding to $\mathbb{P}^{(s)}$ when $s=5$ is shown in  Figure \ref{fig:lattice 10}. Its associated concept lattice  represents the categorization system obtained by considering all the features (financial accounts) in the database, and hence, as it faithfully captures all the information  of $\mathbb{P}^{(s)}$, it is the finest categorization obtainable by any of our proposed methods for this database.

 Consider the set of agents $C= \{j_1,j_2,j_3\}$, and let us assume that agent $j_1$ is interested in the financial accounts $x_1$, $x_2$, and $x_5$, agent $j_2$ in $x_1$, $x_2$, and $x_3$, while agent $j_3$ is interested in $x_1$,  and $x_3$ with various degrees;  their interests  can be represented by the following relation $R \subseteq X^{(s)} \times C$:
  \[
  \begin{array}{r}
  R= \{(x_{1k}, j_1), (x_{2k}, j_1), (x_{5k}, j_1),(x_{1k}, j_2),  (x_{2k}, j_2), (x_{3k}, j_2),  (x_{1k}, j_3), (x_{3k}, j_3) \mid 1 \leq k \leq s                                                          \},
  \end{array}
  \]
which gives rise to the interrogative agendas $Y_i: = R^{-1}[j_i]$ for $1\leq i\leq 3$.   Their associated categorization systems  (i.e.~the concept lattices associated with $\mathbb{P}^{(s)}_{Y_i}$)  are shown in Figure 
 \ref{fig:lattice 1}, \ref{fig:lattice 3}, and \ref{fig:lattice 5}, respectively. Let $c$  be the coalition of $j_1$, $j_2$, and $j_3$. 
Then, the  categorization systems  induced  by   the common agenda $\Diamond c$, and the distributed agenda $\rhd c$ of  coalition $c$ are shown in Figures \ref{fig:lattice 6},  and \ref{fig:lattice 2}, respectively.
  
Let us now assume that   tax is the most relevant account for $j_2$ and $j_2$, while $j_2$ believes that tax and revenues are also relevant, although somewhat less than tax alone, and $j_3$ believes that tax, revenues and expenses are very relevant. Then their interrogative agendas can be represented  by the following mass functions $m_1$, $m_2$, and $m_3$, respectively, for $1 \leq k \leq s$:
 \begin{align*}
   m_1(\{x_{1k}\}) =0.6 \quad m_1(X)=0.4, \\
m_2(\{x_{1k}\}) =0.5 \quad m_2(\{x_{1k},x_{2k}\} )=0.3 \quad m_2(X)=0.2, \\
 m_3\{x_{1k}, x_{2k}, x_{6k}\} =0.9 \quad m_3(X)=0.1, 
 \end{align*}

The most preferred categorization system (i.e.~the one with the highest induced mass) according to $m_1$ and $m_2$ is the same and is shown in   Figure \ref{fig:lattice 6}, while the most preferred categorization according to $m_3$ is shown in  Figure \ref{fig:lattice 7}. For the coalition $c$ of agents $j_1$, $j_2$, and $j_3$, the agendas  $\oplus c=\Diamond c$ and $\rhd c$ are as follows:
\begin{align*}
(\oplus c)(\{x_{1k}\})=0.8 \quad  (\oplus c)(\{x_{1k},x_{2k} \}) =0.12  \quad (\oplus c)(\{x_{1k},x_{2k},x_{6k} \}) =0.072   \quad (\oplus c)(X) = 0.008,\\
(\rhd c)(\{x_{1k},x_{2k},x_{6k} \})=0.432 \quad (\rhd c)(X)=0.568 
\end{align*} 
The most preferred categorization systems according to  the coalition agendas $\oplus c=\Diamond c$, and $\rhd c$ are in Figure \ref{fig:lattice 6}, and Figure \ref{fig:lattice 10}, respectively.

\section{The stability-based method}\label{ssec:stability-based method}
As discussed in Section \ref{sec:Interrogative agendas, coalitions, and categorization},
non-crisp interrogative agendas do not induce a single categorization, but rather, a probability distribution over a parametrized {\em set} of possible categorization systems. 
However, for many applications, it might be desirable to  obtain a {\em single}  categorization system associated with this probability distribution. The simplest  way to define such categorization system is to choose  the concept lattice with the highest preference or probability value attached to it. However, this stipulation ignores a large amount of information of interest in other alternative categorizations. Another possible solution is to estimate the importance attributed to  each feature by a given non-crisp interrogative agenda, using methods such as  plausibility transform \cite{cobb2006plausibility}, pignistic transformation \cite{klawonn1992dynamic, smets2005decision}, or decision probability transformation based on belief intervals \cite{DENG2020106427}.\footnote{The pignistic and plausibility transforms of  the non-crisp agendas {discussed in the previous section} are reported in Appendix \ref{ssec:Importance of different features in these categorizations}.} The values thus obtained, representing the relative importance  of individual features,   can then be used as weights in computing the proximity or dissimilarity between different objects, based on the features shared and not shared between them. The  dissimilarity or proximity data obtained in this way can be used to categorize objects based on any  clustering technique \cite{jain1999data,jain1988algorithms,LI201829}. However, the categorizations obtained with this method are flat (clusterings) rather than hierarchical (concept lattices).
In the present section, we propose a novel {\em stability-based method} for associating a categorization system with any non-crisp agenda.

\smallskip
Throughout the present section, we fix  a  formal context $\mathbb{P} =(A, X,I)$. Recall that, for any non-crisp agenda $m:\mathcal{P}(X) \to [0,1]$, we  let  $m':\mathcal{R} \to [0,1]$ denote its associated  probability mass function on $\mathcal{R} = \{\mathbb{P}^+_Y\mid Y\subseteq X\}$.  
 
 
 \begin{definition}
  For any $G \subseteq A$ s.t.~$G\in \mathbb{P}^+$, the {\em stability index} of $G$ is defined as follows:
 \[
 \rho_m(G) =\sum \{m'(\mathbb{P}^+_Y) \mid G \in  \mathbb{P}^+_Y \text{ for some } Y\subseteq X \}.
 \] 
 \end{definition}
 The value of $ \rho_m(G)$ can be understood as the extent to which  $G$  is a meaningful category according to the non-crisp agenda $m$. Indeed, by construction (cf.~Section \ref{sec:Interrogative agendas, coalitions, and categorization}), $G \in \mathbb{P}^+$ iff $G \in \mathbb{P}_Y^+$ for some $Y\subseteq X$. 
 Hence, intuitively, the higher the weight assigned by $m$ to the categorization systems to which $G$ pertains, the greater the extent to which  $G$ is a meaningful category according to $m$. 
 
 For any $\beta \in [0,1]$ and $m$ as above,
 a  $\beta$-{\em categorization system}  according to $m$ is the complete $\bigcap$-semilattice  (hence complete lattice) $\mathbb{L}(m,\beta)$ of $\mathbb{P}^+$ generated by the set

 \[
\mathcal{C}(m,\beta):=\{ G\in \mathbb{P}^+ \mid G \subseteq A \text{ and } \rho_m(G) \geq \beta\}.
\]
That is,   $\mathbb{L}(m,\beta)$ is the categorization system formed by taking all intersections of  the sets of objects  with stability index greater than  or equal to $\beta$. The  lattice $\mathbb{L}(m,\beta)$  can be understood  as the categorization system which (approximately) represents the given (non-crisp) agenda $m$ given a stability parameter $\beta$. Unlike  the categorization system which is assigned the highest probability by $m$ (if it exists),  this categorization system  incorporates information about other possible categorization systems as well. The  parameter $\beta$ is a `stability threshold' for a concept to be relevant.  

Let $m_1$, $m_2$, and $m_3$ be the agendas of agents $j_1$, $j_2$, and $j_3$
discussed in Section \ref{sec:Example}, and $c$ be the coalition formed by these agents.  The categorization systems associated with $m_1$, $m_2$, and $m_3$ using   the stability-based method   when $\beta=0.5$ and $s=5$  are shown in  Figures \ref{fig:lattice 6}, \ref{fig:lattice 4}, and \ref{fig:lattice 7} respectively, while those associated with  the coalition agendas $\oplus c=\Diamond c$ and ${\rhd}c$   are shown in Figures \ref{fig:lattice 6} and  \ref{fig:lattice 10}, respectively. Note that the categorization systems (lattices) associated with $m_1$ and $\oplus c$ are identical.

\begin{prop}
 For any non-crisp  agenda $m$  and all   $\beta_1, \beta_2 \in [0,1]$, if $\beta_1 \leq \beta_2$, then  $\mathcal{C}(m,\beta_2)\subseteq \mathcal{C}(m,\beta_1)$ and $\mathbb{L}(m,\beta_2)\subseteq \mathbb{L}(m,\beta_1)$.
\end{prop}
 \begin{proof} If  $B\subseteq A$ s.t.~$B \in \mathcal{C}(m,\beta_2)$ $G \subseteq A$ is such that  $\rho_m (G) \geq \beta_2\geq \beta_1$.  Thus, $B  \in \mathcal{C}(m,\beta_1)$. The second inclusion follows from  $ \mathbb{L}(m,\beta_i)$ being meet-generated by $ \mathcal{C}(m,\beta_i)$, for $1\leq i\leq 2$.  
 \end{proof}
  The proposition above matches with the intuition that  lower  values of the stability  threshold  $\beta$ yield  finer-grained categorization systems. 

\begin{remark}
As discussed in Section \ref{sec:Interrogative agendas, coalitions, and categorization}, for any $G \subseteq A$, the set $\mathcal{V}_G:=\{Y \subseteq X\mid G \in \mathbb{P}^+_Y\}$ is upward closed w.r.t.~inclusion; however, $\mathcal{V}_G$ does not need to be closed under intersection. To see this, consider the formal context $\mathbb{P}=(A,X,I)$ with $A=\{a, b\}$, $X=\{x,y,z\}$, and $I=\{(a,x), (a,y), (a,z), (b,y)\}$ and the set $G = \{a\}$;  then $G\in \mathbb{P}^+_Y$ for $Y=\{x,y\}$, and  $G\in  \mathbb{P}^+_Z$ for  $Z=\{y,z\}$,  while $G\notin \mathbb{P}^+_{Y\cap Z}=\mathbb{P}^+_{\{y\}}$. Thus, $\{Y \mid G \in \mathbb{P}^+_Y\}$ does not necessarily have a minimum element. 
\end{remark}

 \begin{definition}\label{def:uprestricted} The {\em upward restricted order} is the partial order $\leq_\uparrow$ on non-crisp agendas defined as follows: 
 $m_1 \leq _\uparrow m_2$ iff for any  subset  $\mathcal{V} \subseteq \mathcal{P}(X) $ which is upward closed w.r.t.~inclusion, 
\[ \sum_{Y \in \mathcal{V}} m_1 (Y)  \leq \sum_{Y \in \mathcal{V}} m_2 (Y). \]
 \end{definition}

 \begin{lemma} \label{lem: rho-order}
If  $m_1 \leq_\uparrow m_2$, then $\rho_{m_1}(G) \leq \rho_{m_2}(G)$ for any $G\in\mathbb{P}^+$. 
 \end{lemma}
 \begin{proof}
 Let $\mathcal{V}_G :=\{Y \subseteq X \mid G \in \mathbb{P}^+_Y\}$. 
 Let $m_1'$ and $m_2'$
 be the mass functions on $\mathcal{R}$ induced by $m_1$, and $m_2$. Then, 
 \[
 \rho_{m_1}(G) = \sum_{Y \in \mathcal{V}_G} m_1'(\mathbb{P}^+_Y)= \sum_{Y \in \mathcal{V}_G} m_1(Y) \leq \sum_{Y \in \mathcal{V}_G} m_2(Y) = \sum_{Y \in \mathcal{V}_G} m_2'(\mathbb{P}^+_Y)= \rho_{m_2}(G),
 \]
where the inequality  above follows from the fact that $\mathcal{V}_G$ is  upward closed, and $m_1 \leq_\uparrow m_2$. 
 \end{proof}

 \begin{prop}\label{prop:specification-categorization} 
   If $m_1 \leq_\uparrow m_2$, then $\mathcal{C}(m_1,\beta)\subseteq \mathcal{C}(m_2,\beta)$ and $\mathbb{L}(m_1,\beta)\subseteq \mathbb{L}(m_2,\beta)$ for any   $\beta \in [0,1]$. 
 \end{prop}
 \begin{proof}
 If $G \in  \mathcal{C}(m_1,\beta)$, then, by the assumption and Lemma \ref{lem: rho-order},   $\beta\leq \rho_{m_1}(G)\leq \rho_{m_2} (G)$, hence 
 ${G} \in \mathcal{C}(m_2,\beta)$. The  second inclusion  follows from the fact that $ \mathbb{L}(m_i,\beta)$, being meet-generated by $ \mathcal{C}(m_i,\beta)$ for $1\leq i\leq 2$.  
 \end{proof}
Therefore, if $m_1 \leq _\uparrow m_2$, then, for any fixed stability parameter $\beta \in [0,1]$, the categorization system induced by $m_1$ using the stability-based method  is coarser than the one induced $m_2$. Hence,  a  $ \leq _\uparrow $-smaller agenda considers less  information  relevant for categorization. 

 \section{Orderings of mass functions and coalition agendas}
 \label{sec: orderings}
Several orderings on Dempster-Shafer mass functions have been introduced in the literature, based e.g.~on plausibility and quality functions, among others.  In the present section, we compare the ordering  $\leq_\uparrow$ defined in the previous section  with these orderings. 
 \begin{definition}[{\cite[Section 2.3]{denoeux2006cautious}}]\label{def:mass-ordering}
 For any $m_1, m_2: \mathcal{P}(X) \to [0,1]$, 
 \begin{enumerate}
     \item {\em pl-ordering}:  $m_1 \leq_{{pl}} m_2$  iff  $pl_{m_1}({Y}) \leq pl_{m_2}({Y})$ for every ${Y} \subseteq X$.
     \item  {\em q-ordering}:  $m_1 \leq_{{q}} m_2$ iff $q_{m_1}({Y}) \leq q_{m_2}({Y})$ for every ${Y} \subseteq X$.
     \item  {\em s-ordering}:  $m_1 \leq_{\mathrm{s}} m_2$ iff some square matrix $S: \mathcal{P}(X)\times \mathcal{P}(X)\to [0, 1]$ exists such that:
 \begin{enumerate}
     \item $\sum_{{W} \subseteq X} S(W,{Y}) =1$ for every $Y\subseteq  X$;
     \item for all $ W, Y \subseteq X$, if $S(W,{Y}) > 0$ then $W \subseteq {Y}$;
     \item $m_1(W)= \sum_{Y \subseteq X}S(W,Y)m_2({Y})$ for any $ W \subseteq X$. 

\end{enumerate}
     In this case, $m_1$ is said to be a {\em specialization} of $m_2$, since the value that $m_1$ assigns to a given set $W$ is the sum of all the individual fractions  $S(W,{Y})$  of the values assigned by $m_2$ to every superset $Y$ of $W$. The matrix $S$ is called a {\em specialization matrix}. 
     
     \item {\em Dempsterian specialization ordering}:  $m_1 \leq_{\mathrm{d}} m_2$ iff
     $m_1= m \cap_m m_2$ for some Dempster-Shafer mass function $m$, where $m_1 \cap m_2$ denotes the un-normalized Dempster's combination given by
     \begin{equation} \label{eq: un-normalized DS rule}
          m_1 \cap_m m_2({Y})= \sum_{{Y}_1 \cap {Y}_2={Y}}m_1({Y}_1)m_2({Y}_2).
     \end{equation}
 \end{enumerate}
  \end{definition}

 It is well known that \cite[Section 2.3]{denoeux2006cautious}
 \begin{equation} \label{eq:mass-ordering implications}
   m_1 \leq_d m_2 \implies m_1 \leq_s m_2 \implies \begin{cases} m_1 \leq_{{pl}} m_2 \\
  m_1 \leq_q m_2.
  \end{cases}
 \end{equation}

 \begin{prop}
 \label{prop:various ordering compared}
 For all mass functions $m_1, m_2: \mathcal{P}(X) \to [0, 1]$
  \begin{equation} \label{eq: new mass-ordering implications}
   m_1 \leq_d m_2 \implies m_1 \leq_s m_2\implies m_1 \leq _\uparrow m_2 
   \implies \begin{cases} m_1 \leq_{{pl}} m_2 \\
  m_1 \leq_q m_2.
  \end{cases}
 \end{equation}

 \end{prop}
 \begin{proof}
 We only need to prove the implications involving $\leq _\uparrow$. 
 Let us assume that $m_1 \leq_s m_2$, i.e.~a square matrix $S$ exists satisfying conditions (a)-(c) of Definition \ref{def:mass-ordering}.3. 
 %
Let $\mathcal{V} \subseteq\mathcal{P}(X) $ be upward closed. Then
\begin{center}
    \begin{tabular}{rcll}
     $\sum_{W\in\mathcal{V}}m_1(W)$& $=$& $\sum_{W\in\mathcal{V}} \sum_{W\subseteq Y}S(W,Y)m_2({Y})$ & (c) and (b) \\ &$=$&$ \sum_{Y\in\mathcal{V}}\sum_{W\subseteq Y, W\in\mathcal{V}}S(W,Y)m_2(Y)$ & $\mathcal{V}$ upward closed\\
    & $\leq$ & $\sum_{Y\in\mathcal{V}}m_2(Y)$, & (a)\\
       \end{tabular}
\end{center}
which shows that $m_1\leq_{\uparrow} m_2$. The inclusions $\leq_\uparrow\ \subseteq \ \leq_q$ and $\leq_\uparrow\ \subseteq \ \leq_{pl}$
follow immediately from the fact that the sets $\{Z \subseteq  X \mid Y \subseteq Z\}$  and $\{Z \subseteq  X \mid Y \cap  Z \neq \varnothing \}$ are upward closed w.r.t.~inclusion. 
 \end{proof}

 \begin{remark} 
 The  following examples show that the inclusions $\leq_\uparrow\ \subseteq \ \leq_q$ and $\leq_\uparrow\ \subseteq \ \leq_{pl}$  are strict.   Let $X=\{y_1, y_2,y_3\}$, and $m_1$ and $m_2$ be the following mass functions on $X$:
 \[
 m_1(\{y_1,y_3\})=0.3, \quad m_1(\{y_2,y_3\})=0.3, \quad m_1(\{y_1,y_2,y_3\})=0.2, \quad m_1(\{y_3\})=0.2\quad \text{and}
 \]
 \[
 m_2(\{y_1,y_3\})=0.1, \quad m_2(\{y_2,y_3\})=0.1, \quad m_2(\{y_1,y_2,y_3\})=0.5, \quad m_2(\{y_3\})=0.3.
 \]
 The following table reports the values of  $q_{m_1}$ and $q_{m_2}$ for all  $Y \subseteq X$.

\begin{center}  
 \begin{tabular}{|c|c|c|c|c|c|c|c|c|}
 \hline
     & $\varnothing$ & $\{y_1\}$ &  $\{y_2\}$ &  $\{y_3\}$ & $\{y_1,y_2\}$ &  $\{y_2,y_3\}$ &  $\{y_1, y_3\}$ &  $\{y_1, y_2, y_3\}$\\
     \hline
      $q_{m_1}$& 1 & 0.5 & 0.5 & 1 & 0.2 & 0.5 &0.5 & 0.2 \\
      \hline
       $q_{m_2}$& 1& 0.6 & 0.6 &  1 & 0.5 & 0.6 & 0.6 & 0.5 \\
       \hline 
 \end{tabular}
\end{center}
Therefore,  $m_1 \leq_q m_2$. However, for the upward closed set $\mathcal{V} =\{\{y_1,y_3\}, \{y_2,y_3\}, \{y_1,y_2,y_3\}\}\subseteq \mathcal{P}(X)$, we have
  \[\sum_{Y \in \mathcal{V}} m_2 (Y)  < \sum_{Y \in \mathcal{V}} m_1  (Y).\] 
Let $m_3$ and $m_4$ be the following mass functions on $X$:
\[
 m_3(\{y_1,y_2\})=0.3, \quad m_3(\{y_2,y_3\})=0.4, \quad m_3(\{y_1,y_3\})=0.3, \quad \text{and}
 \]
\[
 m_4(\{y_1\})=0.1, \quad m_4(\{y_2\})=0.2, \quad m_4(\{y_3\})=0.2, \quad m_4(\{y_1,y_2\})=0.5.
\]

 The following table reports the values of  $pl_{m_3}$ and $pl_{m_4}$ for all  $Y \subseteq X$.

\begin{center}  
 \begin{tabular}{|c|c|c|c|c|c|c|c|c|}
 \hline
     & $\varnothing$ & $\{y_1\}$ &  $\{y_2\}$ &  $\{y_3\}$ & $\{y_1,y_2\}$ &  $\{y_2,y_3\}$ &  $\{y_1, y_3\}$ &  $\{y_1, y_2, y_3\}$\\
     \hline
      $pl_{m_3}$& 0 & 0.6 & 0.7 & 0.7 &1 & 1 &1 & 1\\
      \hline
       $pl_{m_4}$& 0& 0.6 & 0.7 &  0.2 & 0.8 & 0.9 & 0.8 & 1 \\
       \hline 
 \end{tabular}
\end{center}
Therefore,  $m_4 \leq_{pl} m_3$. However, for the upward closed set $\mathcal{V} =\{ \{y_1,y_2\}, \{y_1,y_2,y_3\}\}\subseteq \mathcal{P}(X)$, we have
  \[\sum_{Y \in \mathcal{V}} m_3 (Y)  < \sum_{Y \in \mathcal{V}} m_4  (Y).\] 
We leave it as an open question to check if the converse of implication between $\leq_s$ and $\leq_\uparrow$ holds.
 \end{remark}

As an immediate consequence of  Lemma \ref{lem: rho-order} and Propositions \ref{prop:specification-categorization} and \ref{prop:various ordering compared}, we get the following
 \begin{cor} \label{cor:Dempster order}
 If  $m_1 \leq_s m_2$ or $m_1 \leq_d m_2$,  then $\rho_{m_1}(G) \leq \rho_{m_2}(G)$ for any $G\in\mathbb{P}^+$, and  $\mathcal{C}(m_1,\beta)\subseteq \mathcal{C}(m_2,\beta)$ and $\mathbb{L}(m_1,\beta)\subseteq \mathbb{L}(m_2,\beta)$ for any   $\beta \in [0,1]$.
 \end{cor}

 \begin{lemma} \label{lem: combination-ordering relation}
 Let $C$ be a set of agents. For any coalition  $c\subseteq C$ and any agent $j\in c$,  
 \[\Diamond  c \leq_s m_j\quad
\text{ and  }\quad    m_j  \leq_\uparrow  \rhd c.\]
 \end{lemma}
 \begin{proof}

Let us show the first inequality by induction on the cardinality of  $c$. The base case in which $c$ is a singleton is trivial. 
For the inductive step, note that if $c' = c \cup \{i\}$, then $\Diamond (c')= \Diamond  c \cap_m m_i$. Thus,  it is enough to show that for   any two mass functions $m_1$ and $m_2$, $m_1 \cap_m m_2 \leq_s m_1$. 
The proof follows by setting
\[
S(W,Y)= \sum \{m_2(Y') \mid Y' \cap Y = W  \}.
\]
It is straightforward to check that $S$ satisfies all the conditions required in Definition \ref{def:mass-ordering}.

Let us show the second inequality by induction on the cardinality of  $c$. The base case in which $c$ is a singleton is trivial. 

For any mass functions $m_1$ and $m_2$, let $m_1 \cup_m m_2$ be the mass defined as follows: for any $Y\subseteq X$, 
$m_1 \cup_m m_2(Y)=\sum_{Z_1 \cup Z_2 =Y}m_1(Z_1)m_2(Z_2)$. For the inductive step,  if $c' = c \cup \{i\}$, then $\rhd (c')= (\rhd c )\cup_m m_i$. Thus,  it is enough to show that  $m_1 \leq_s m_1 \cup_m m_2 $ for   all mass functions $m_1$ and $m_2$. 

For any $Y \subseteq X$, the mass $m_1(Y)$ is transferred completely to the sets larger than or equal to $Y$ in performing the operation $\cup$. Thus, any mass assigned by $m_1$ to any set in an up-set $\mathcal{V}$ remains in $\mathcal{V}$ in $m_1 \cup_m m_2$. In particular, for every $Y\in\mathcal{V}$ and every $Y'\subseteq X$, $Y\cup Y'\in\mathcal{V}$. Therefore, 
 \[\sum_{Y \in \mathcal{V}} m_1 (Y) =\sum_{Y \in \mathcal{V}}(\sum_{Y'\subseteq X}m_2(Y'))m_1(Y)  \leq 
 \sum_{Y \in \mathcal{V}}\sum_{Y_1\cup Y_2=Y}m_1(Y_1)m_2(Y_2)= \sum_{Y \in \mathcal{V}} (m_1 \cup_m m_2) (Y).\]

 \end{proof}

The next corollary follows immediately from Lemma \ref{lem: combination-ordering relation} and Proposition \ref{prop:specification-categorization}.

 \begin{cor} \label{cor:intersection-union masses}
 Let $C$ be a set of agents. For any coalition  $c\subseteq C$, any agent $j\in c$, and any $\beta \in [0,1]$, \[\mathcal{C}( \Diamond  c,\beta)\subseteq  \mathcal{C}(m_j,\beta) \subseteq \mathcal{C}( \rhd c,\beta)\quad \text{  and }\quad \mathbb{L}( \Diamond c,\beta)\subseteq\mathbb{L}( m_j ,\beta)\subseteq \mathbb{L}( \rhd c,\beta).\]

 \end{cor}
Hence, for any coalition $c$ and any  parameter $\beta$, the categorization system generated by the agendas $\Diamond c$  (resp.~$\rhd c$) using the stability-based method with parameter $\beta$ is coarser (resp.~finer) than  the categorization system generated by the agenda $m_j$ of any agent $j \in c$. This can further justify understanding  the non-crisp agenda $\Diamond c$  (resp.~$\rhd c$) as the non-crisp counterpart of    $\Diamond c$  (resp.~$\rhd c$) as defined in \eqref{eq:Diamond-triangle}.

 \section{Meta-learning algorithm using interrogative agendas}\label{sec:algorithms}
 The extant outlier detection and classification algorithms based on FCA \cite{fu2004comparative,prokasheva2013classification,kuznetsov2013fitting,sugiyama2013semi,zhang2014outlier} are typically explainable, being based on data-analytic methods,   
 and determine an appropriate set of relevant features using some feature-selection methods. However, unlike the methods discussed in the previous sections, most feature-selection methods do not allow for different (sets of) features  to be  assigned different relevance. 
 In \cite{acar2023meta},   a meta-algorithm is introduced which
learns the (non-crisp) agenda which is best suited for given outlier detection or classification tasks. Given some  outlier detection or classification algorithm in input, which acts on a single 2-valued concept lattice, the meta-algorithm    applies it to all the concept lattices induced by the different agendas under consideration. 
The output of the meta-algorithm is obtained by combining the predictions associated with each of these concept lattices by means of some  aggregation function (e.g.~weighted average, maximum, minimum) which   weights  the different lattices according to  their relevance values. The best values for these weights relative to a given task can then  be learned  from training data using gradient descent. The following is an edited version of the algorithm in \cite{acar2023meta}, in which the specific aggregation function used in \cite{acar2023meta} is replaced by a generic one.

 \begin{algorithm}
\footnotesize
\caption{Meta-Learning Algorithm for Interrogative Agendas} \label{algo:metaxai}
\hspace*{\algorithmicindent} \textbf{Input:} - A formal context $\mathbb{P} = (A, X, I)$,\\
\hspace*{\algorithmicindent} \phantom{\textbf{Input:}} - a training set $T\subseteq A$, a map $f:T \to Lab$ assigning labels to the elements of   the training set, \\
\hspace*{\algorithmicindent} \phantom{\textbf{Input:}} - a set $\mathcal{Z} \subseteq\mathcal{P}(X)$ consisting of $n$  different agendas under consideration, \\
\hspace*{\algorithmicindent} \phantom{\textbf{Input:}} - an algorithm $Alg$ which takes  an object $a\in A$ and a concept lattice in $\mathcal{R} = \{\mathbb{P}^+_Y\mid Y\subseteq X\}$ in input, and outputs an element in \\
\hspace*{\algorithmicindent} \phantom{\textbf{Input:} - } $\mathbb{R}^{Lab}$ representing, for each $l\in Lab$,  the prediction of $Alg$ on the likelihood that  $l$ be assigned  to $a$; \\
\hspace*{\algorithmicindent} \phantom{\textbf{Input:}} - a loss function $loss: (\mathbb{R}^{Lab})^{|T|}\times (\mathbb{R}^{Lab})^{|T|} \to \mathbb{R}$, \\
\hspace*{\algorithmicindent} \phantom{\textbf{Input:}} - an aggregation function $Agg:({\mathbb{R}^{Lab}})^n \times \mathbb{R}^n \to \mathbb{R}^{Lab}$ which combines the outputs of all agendas given some weights $\overline{w} \in \mathbb{R}^{n}$,\\
\hspace*{\algorithmicindent} \phantom{\textbf{Input:}} - a number of training epochs $epochs$.\\
\hspace*{\algorithmicindent} \textbf{Output:} A model that classifies objects in $A$.
\begin{algorithmic}[1]
\Procedure{Train}{$\mathbb{P}$,  $T$, $f$, $Alg$, $Agg$, $loss$, $epochs$}
    \State $\mathbb{L}_1,\ldots,\mathbb{L}_n \leftarrow $ \textbf{compute} the concept lattices induced by the agendas in $\mathcal{Z}$
    \State \textbf{let} $predictions$ be an empty map from $A$ to $\mathbb{R}^C$
    \State \textbf{let} $\overline{w}$ be an array of random weights of length $n$
    \For{$e = 1, \ldots, epochs$ } 
        \For{$a \in A$, $k \in C$}
            \State $scores \leftarrow \big(Alg_k(a, \mathbb{L}_1),Alg_k(a, \mathbb{L}_2), \cdots, Alg_k(a, \mathbb{L}_n )\big)$
            \State $predictions[a][k]\leftarrow Agg(scores, \overline w)$
        \EndFor
        \State \textbf{update} $\overline{w}$ with an iteration of gradient descent (use $loss$)
    \EndFor
\EndProcedure
\end{algorithmic}
\end{algorithm}

As the cardinality of $\mathcal{R}$ is $2^{|X|}$, 
only mass functions with a limited number of focal elements are considered, which are 
determined by the objectives and constraints of the algorithm. For simplicity, the weights  in the learning phase are allowed to take values in $\mathbb{R}$. Interrogative agendas (mass functions) can be obtained by normalizing the absolute values of these weights. 

 The weights learned by the algorithm describe the importance of different sets of features relative to a given task. These weights  can be understood as a generalization of Shapely values \cite{molnar2020interpretable}, which estimate the importance of individual features in given machine learning tasks.

As  agendas  are represented as  Dempster-Shafer mass functions,  tools from Dempster-Shafer theory are available for furthering their theory and applications. For example,  
 agendas learned by different algorithms can be combined by means of different aggregation rules from Dempster-Shafer theory, so to obtain an ensemble algorithm; or agendas can be compared or aggregated with the agendas of human experts,  to allow for  human-machine collaboration. This is especially important in fields such as auditing, where human supervision of algorithms is necessary for ethical and legal reasons.

The weights learned  by the meta-algorithm arguably provide {\em global explainability}, since they encode the relevance of different sets of features in the classification decision of the algorithm.  Furthermore, the meta-algorithm is {\em locally explainable}, since its output for any object can be described and tracked back in terms of the individual outputs obtained by running the algorithm $Alg$ on the concept lattices induced by each agenda under consideration with the given input.

An implementation of this meta-algorithm for outlier detection, its comparison with other outlier detection algorithms, and the explanations provided by it are discussed in detail in \cite{boersma2023outlier}.  In the future, we intend to carry out a similar study for the classification task. 


 \section{Conclusion and further directions} \label{sec:Conclusion and further directions}
 \paragraph{Main contributions.}  The present paper introduces a {\em multi-agent framework} for describing and reasoning about the space of categorization systems over a given set of entities. This framework combines formal tools from FCA (namely formal contexts and their associated concept lattices) which are used to represent categorization systems, 
and tools from Dempster-Shafer theory (namely mass functions) which serve to encode the agents' epistemic attitudes (non-crisp interrogative agendas).

We illustrate this framework with a {\em case study} focusing on the problem of categorizing business processes for auditing purposes. 
 We  introduce a  {\em stability-based method} for associating a single categorization system (complete lattice) with a Dempster-Shafer mass function on the set of features of a formal context; we define the {\em upward restricted order} on mass functions (cf.~Definition \ref{def:uprestricted}), with respect to which the mapping of categorization systems to mass functions defined by  the stability-based method is order preserving (cf.~Proposition \ref{prop:specification-categorization}), and compare this order with other orderings on mass functions introduced in the literature.
  Finally, we introduce a {\em meta-learning algorithm} for various classification tasks which can provide both global and local explanations of  the outcomes of the algorithms to which it is applied.

%
   The generality of these  contributions makes them applicable  to classification problems far beyond 
   financial transactions. 
 Below, we discuss  some  directions for future research.



\paragraph{Implementing the meta-learning algorithm.}
In \cite{boersma2023outlier},   the meta-algorithm introduced in Section
\ref{sec:algorithms}  is implemented, by using a specific outlier detection algorithm $Alg$ as its input. 
In the future, we intend to study several such implementations, using various   FCA-based outlier detection and classification algorithms such as those in  
\cite{fu2004comparative,prokasheva2013classification,kuznetsov2013fitting,sugiyama2013semi,zhang2014outlier} and compare their performances. 

Interestingly, this study would also allow one  to {\em compare the weights} \cite{jousselme2001new} learned by the meta-algorithm with different algorithms in input, and hence to understand 
whether different algorithms assign  similar importance to a given set of features while performing the same task.

Moreover, agendas obtained from different algorithms can  be  {\em aggregated} using Dempster-Shafer based methods \cite{BI20081731,sentz2002combination}, and likewise,  agendas of {\em human experts} can  be combined or compared with the agendas learned by the meta-algorithm.

Another interesting direction for future research is to extend the meta-algorithm \ref{algo:metaxai} to {\em other tasks} where formal concept analysis has been successfully applied, such as data mining, information retrieval, attribute exploration,  knowledge management \cite{priss2006formal,qadi2010formal,poelmans2013formal,ganter2012formal, wille1996formal}. 


Specific instantiations of the meta-learning algorithm can often be fine-tuned to more efficient procedures. For example, the outlier detection algorithm proposed in \cite{boersma2023outlier} only computes the formal concepts generated by objects, rather than  the   full concept lattices.  
This and other similar fine-tunings  
can mitigate the  relatively high computational costs of these algorithms, while still adopting the fundamental idea laid of the meta-algorithm. In the future, we will also devote attention to efficiency procedures.


 
\subsection*{Declaration of interest and disclaimer} The authors report no conflicts of interest, and declare that they have no relevant or material financial interests related to the research in this paper. The authors alone are responsible for the content and writing of the paper, and the views expressed here are their personal views and do not necessarily reflect the position of their employer.

\bibliographystyle{plain}
\bibliography{ref}

 \appendix
 
\section{Financial statements network example} \label{sec:dataset}
In this section, we describe the (toy) financial statements network used to build the formal context discussed in Section \ref{sec:Example}, and the estimated importance values of different features according to the different agendas encountered throughout the paper obtained using the pignistic transform and plausibility transform. 

\subsection{Database and its associated formal context}
 Table \ref{database table} is a small database showing different transactions and financial accounts used to obtain a small financial statements network considered in the examples.
 
{
\singlespacing
 \begin{table}[h]
 \begin{tabular}{|c|c|l|c|} 
 \hline
 \textbf{ID} & \textbf{TID} & \textbf{FA name} &  \textbf{Value}\\
 \hline
 1 & 1& revenue & -100\\
 \hline
 2& 1& cost of sales& +100\\
 \hline
 3 & 2& revenue & -400\\
 \hline
 4 & 2& personal expenses  & +400\\
 \hline
5 & 3& other expenses  & +125\\
\hline
 6 & 3& cost of sales  & +375\\
 \hline
 7 & 3& revenue & -500\\
 \hline
 8 & 4& tax  & -125\\
 \hline
 9 & 4& cost of sales  & +500\\
 \hline
 10 & 4& revenue & -375\\
 \hline
  11 & 5& tax  & -10\\
  \hline
 12 & 5& cost of sales  & +200\\
 \hline
 13 & 5& revenue & -190\\
 \hline
 14 & 6& other expenses  & +50\\
 \hline
 15 & 6& cost of sales  & +450\\
 \hline
 16 & 6& inventory & -500\\
 \hline
\end{tabular}
 \begin{tabular}{|c|c|l|c|} 
 \hline
 \textbf{ID} & \textbf{TID} & \textbf{FA name} &  \textbf{Value}\\
 \hline
 17 & 7 & cost of sales  & +400\\
 \hline
 18 & 7 & revenue  & -300\\
 \hline
 19 & 7 & inventory  & -100\\
 \hline
 20 & 8& revenue & -150\\
 \hline
 21& 8& cost of sales & +150\\
 \hline
 22 & 9& revenue & -250\\
 \hline
 23& 9& cost of sales &+250\\
 \hline
  24 & 10& tax & -250\\
  \hline
 24& 10& personal expenses & +250\\
 \hline
  26 & 11& revenue & -250\\
  \hline
 27& 11& personal expenses & +175\\
 \hline
 28& 11& other  expenses & +75\\
 \hline
  29 & 12& revenue & -250\\
  \hline
  30 & 12& tax & -50\\
  \hline
 31& 12& personal expenses  &+150\\
 \hline
 32& 12& other  expenses  &+150\\
 \hline
  \end{tabular}
 \caption{A small database with 12 business processes and 6 financial accounts. We use same TID to denote all credit and debit activities relating to a single business process.}
 \label{database table}
  \end{table}
  }
 \smallskip
Table  \ref{table: context table}  shows the many valued context obtained by interpreting business processes as objects and financial accounts as features. The value of the incidence relation denotes the share of given financial account in a given business process. 

{\singlespacing
\centering
\begin{table}[H]
   \begin{tabular}{|c|l|c|}
 \hline 
\textbf{Business }&\,\,\textbf{Financial}& \textbf{Share of} \\
\textbf{ process} \textbf{(a)}&\,\,\textbf{ account}\textbf{ (x)}& \textbf{value} \textbf{(I(a,x))}\\
  \hline
   1 ($a_1$)& revenue ($x_2$)  & -1\\
   \hline
  1 ($a_1$)& cost of sales ($x_3$)& +1\\
  \hline
  2 ($a_2$)& revenue ($x_2$)& -1\\
  \hline
 2 ($a_2$) & personal expenses  ($x_4$)& +1\\
 \hline
 3 ($a_3$)& other expenses ($x_6$) & +0.25\\
 \hline
  3 ($a_3$)& cost of sales ($x_3$) & +0.75\\
  \hline
 3 ($a_3$) & revenue ($x_2$)& -1\\
 \hline
 4 ($a_4$) & tax  ($x_1$) & -0.25\\
 \hline
 4 ($a_4$)& cost of sales ($x_3$) & +1\\
 \hline
 4 ($a_4$)& revenue ($x_2$) & -0.75\\
 \hline
 5 ($a_5$)& tax  ($x_1$) & -0.05\\
 \hline
 5 ($a_5$)& cost of sales ($x_3$) & +1\\
 \hline
 5 ($a_5$)& revenue ($x_2$)& -0.95\\
 \hline
 6 ($a_6$)& other expenses ($x_6$) & +0.1\\
 \hline
6 ($a_6$) & cost of sales ($x_3$)   & +0.9\\
\hline
 6 ($a_6$)& inventory ($x_5$)& -1\\
 \hline
\end{tabular}
   \begin{tabular}{|c|l|c|}
 \hline 
\textbf{Business }&\,\,\textbf{Financial}& \textbf{Share of} \\
\textbf{ process} \textbf{(a)}&\,\,\textbf{ account}\textbf{ (x)}& \textbf{value} \textbf{(I(a,x))}\\
  \hline
 7 ($a_7$)& cost of sales ($x_3$) & +1\\
 \hline
  7 ($a_7$)& inventory ($x_5$) & -0.25\\
  \hline
 7 ($a_7$)& revenue  ($x_2$)& -0.75\\
 \hline
 8 ($a_8$) & revenue ($x_2$)& -1\\
 \hline
  8 ($a_8$)& cost of sales ($x_3$)& +1\\
  \hline
 9 ($a_9$)& revenue ($x_2$)& -1\\
 \hline
9 ($a_9$)& other expenses ($x_6$)&+1\\
\hline
  10 ($a_{10}$)& tax ($x_1$) & -1\\
  \hline
 10  ($a_{10}$) & personal expenses ($x_4$) & +1\\
 \hline
 11  ($a_{11}$)& revenue($x_2$) & -1\\
 \hline
 11 ($a_{11}$)& personal expenses ($x_4$) & +0.7\\
 \hline
 11 ($a_{11}$)& other  expenses ($x_6$) & +0.3\\
 \hline
 12 ($a_{12}$)& revenue($x_2$) & -0.83\\
 \hline
 12  ($a_{12}$)& tax($x_1$)  & -0.17\\
 \hline
12  ($a_{12}$)& personal expenses ($x_4$)  &+0.5\\
\hline
 12  ($a_{12}$)& other  expenses ($x_6$)  &+0.5\\
 \hline
\end{tabular}
\caption{The formal context obtained from transaction database given in Table \ref{database table}. $I(a,x)=0$ for any $(a,x)$ pair not present in the table.}
\label{table: context table}
\end{table}
}

 \subsection{Pignistic and plausibility transforms} \label{ssec:Importance of different features in these categorizations}
 Tables \ref{tab:pignistic} and \ref{tab:plausibility} report the estimated values of the importance of different features (financial accounts) calculated via pignistic and plausibility transformations  of  the non-crisp agendas (mass functions) discussed in Section \ref{sec:Example}.  

 \begin{table}[H]
 {
     \centering
     {%
     \begin{tabular}{|c|c|c|c|c|c|c|}
     \hline
        Agenda & $x_1$&$x_2$& $x_3$&$x_4$&$x_5$&$x_6$\\
    \hline
          $m_1$ & 0.67 & 0.067&0.067&0.067&0.067&0.067\\
    \hline 
          $m_2$ & 0.683 & 0.183&0.033&0.033&0.033&0.033\\
    \hline 
          $m_3$ & 0.317 & 0.317&0.017&0.017&0.017&0.317\\
    \hline 
          $m$ & 0.885 & 0.085&0.001&0.001&0.001&0.025\\
    \hline 
          $m'$ & 0.239 & 0.239&0.095&0.095&0.095&0.239\\
    \hline 
     \end{tabular}%
     }
     \caption{Importance estimates via pignistic transformation}\label{tab:pignistic}
     }
        
 \end{table} 
 \begin{table}[H]
 {
     \centering
     {%
     \begin{tabular}{|c|c|c|c|c|c|c|}
     \hline
        Agenda & $x_1$&$x_2$& $x_3$&$x_4$&$x_5$&$x_6$\\
    \hline
          $m_1$ & 0.333 & 0.133&0.133&0.133&0.133&0.133\\
    \hline 
          $m_2$ & 0.435 & 0.217 &0.087&0.087&0.087&0.087\\
    \hline 
          $m_3$ & 0.303 & 0.303&0.030&0.030&0.030&0.303\\
    \hline 
          $m$ & 0.767 & 0.153&0.006&0.006&0.006&0.061\\
    \hline 
          $m'$ & 0.213 & 0.213&0.121&0.121&0.121&0.213\\
    \hline 
     \end{tabular}%
     }
     \caption{Importance estimates via plausibility transformation}\label{tab:plausibility}
     }
 \end{table}

 \section{Concept lattices associated with various interrogative agendas}\label{sec:lattice diagrams}
 
 The present appendix section collects the Hasse diagrams, drawn with the help of  LatViz \cite{Latviz, alam2016latviz}, of various concept lattices associated with the interrogative agendas in Section \ref{sec:Example}.

 \begin{figure*}
    \centering
    \includegraphics [width=1.0\textwidth]{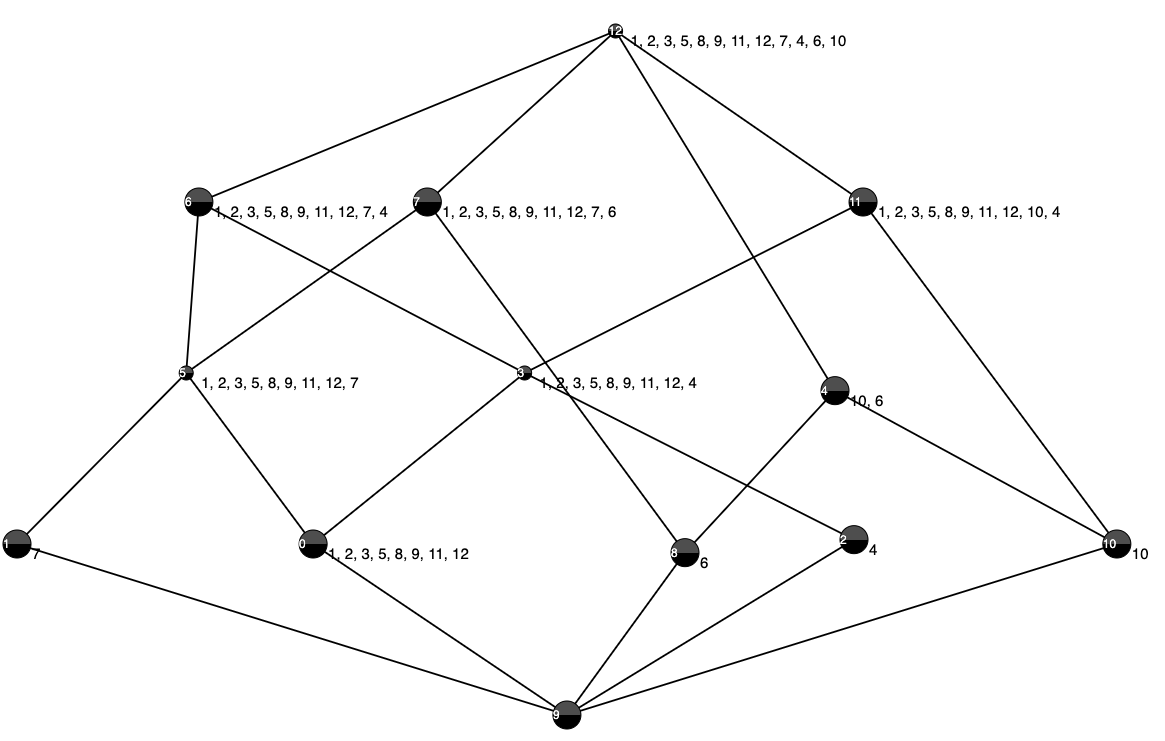}
    \caption{Concept lattice associated with the  crisp agenda $\{x_1, x_2, x_5\}$ of agent $j_1$}
    \label{fig:lattice 1}
\end{figure*}

\begin{figure*}
    \centering
    \includegraphics [width=1.0\textwidth]{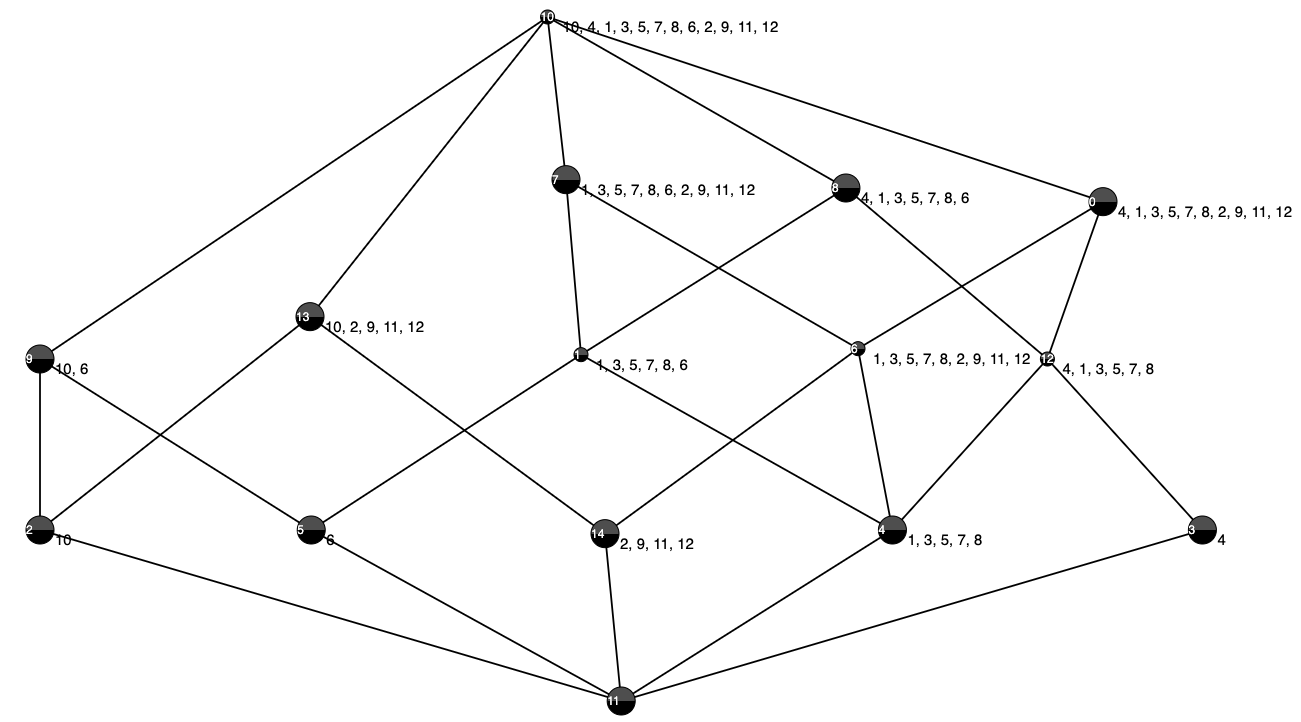}
    \caption{Concept lattice associated with the crisp agenda  $\{x_1, x_2, x_3\}$ of agent $j_2$}
    \label{fig:lattice 3}
\end{figure*}

\begin{figure*}
    \centering
    \includegraphics [width=1.0\textwidth]{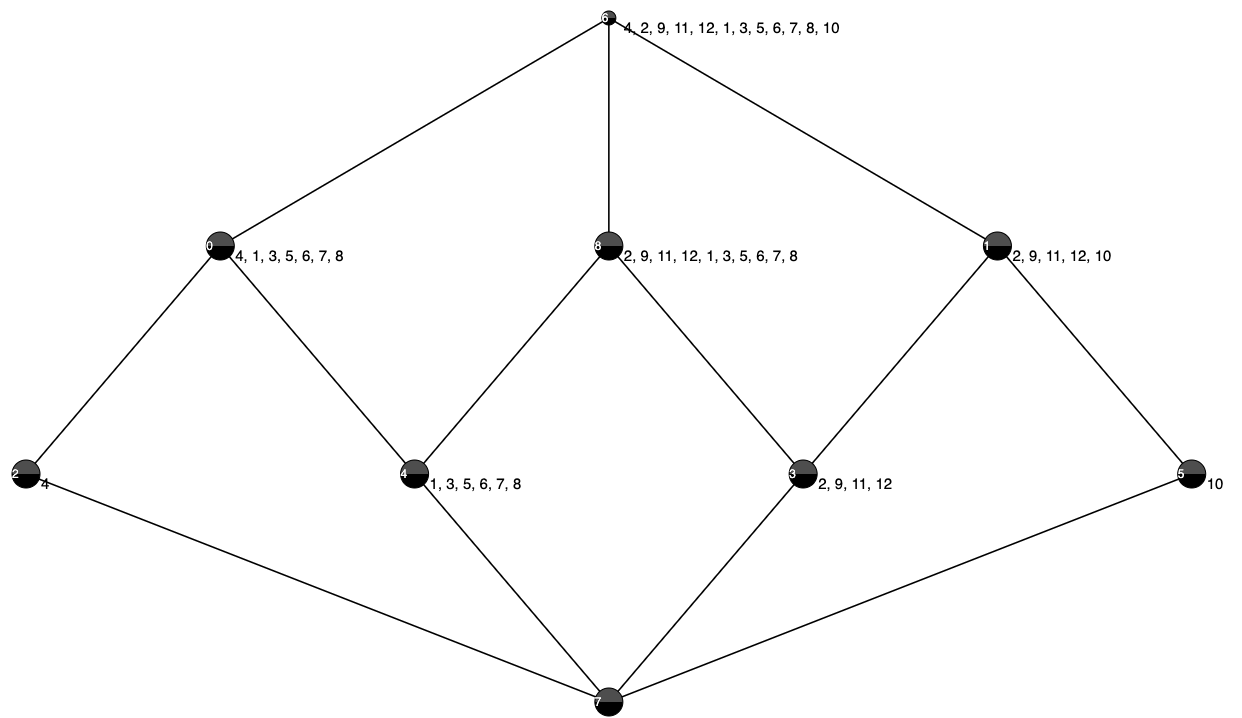}
    \caption{Concept lattice associated with the crisp agenda  $\{x_1, x_3\}$ of agent $j_3$}
    \label{fig:lattice 5}
\end{figure*}

\begin{figure*}
    \centering
    \includegraphics[width=1.0\textwidth]{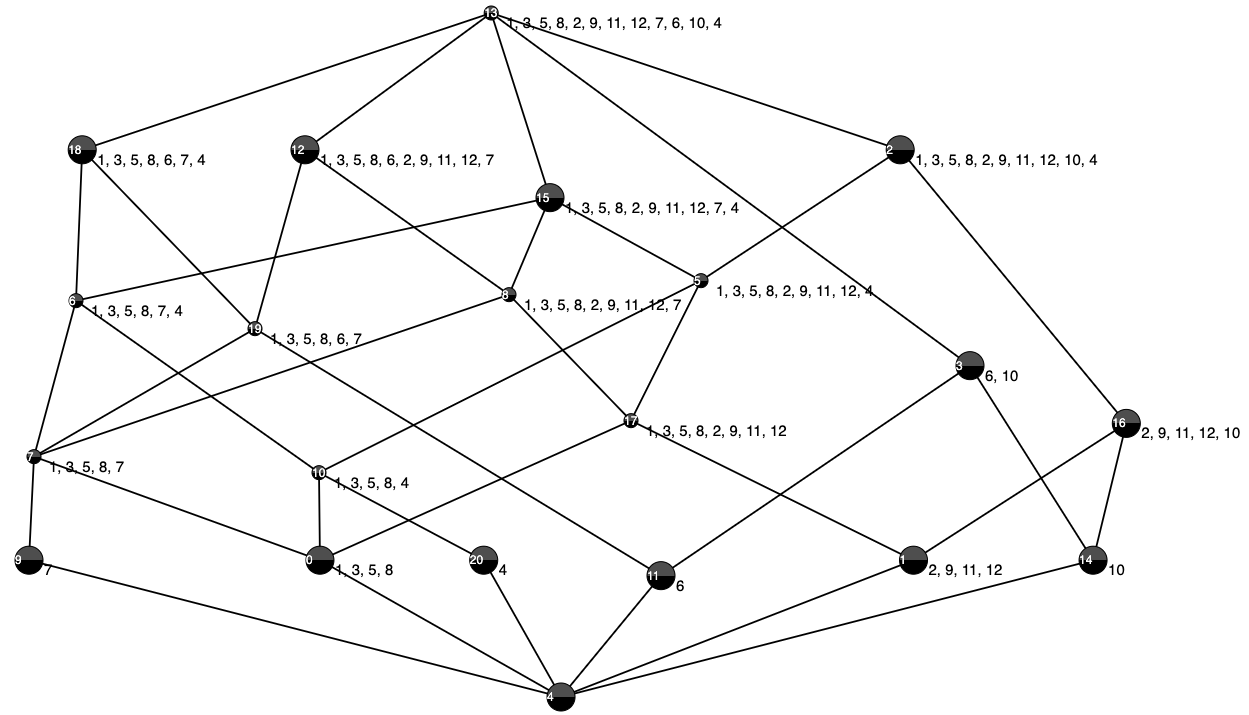}
    \caption{Concept lattice associated with the crisp  agenda  $\rhd c =\{x_1,x_2,x_3,x_5\}$}
    \label{fig:lattice 2}
\end{figure*}

\begin{figure*}
    \centering
    \includegraphics[width=1.0\textwidth]{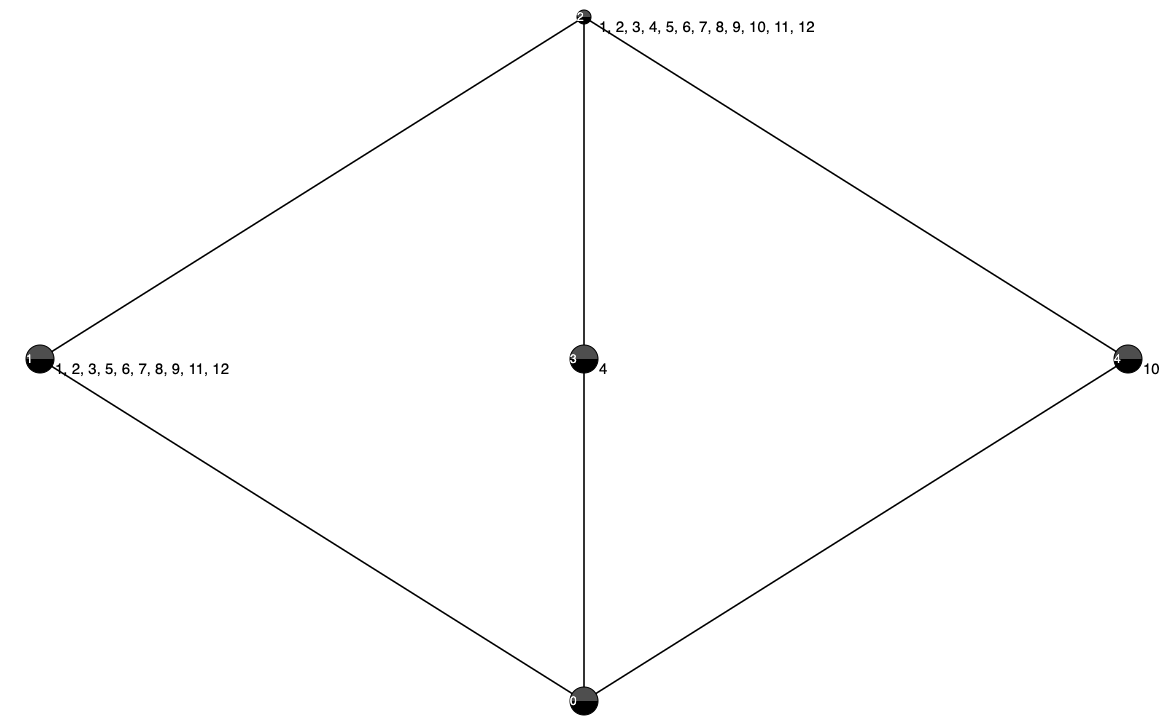}
    \caption{
    This concept lattice is all of the following:   (1) the concept lattice associated with   the crisp agenda $\Diamond c=\{x_1\}$ in case, 
    (2) most preferred categorization associated with the non-crisp agendas $m_1$ of $j_1$, $m_2$ of $j_2$, and $\oplus c$ (3) categorization obtained from non-crisp agendas $m_1$ and $\oplus c$ using stability-based method for $\beta=0.5$. } 
    \label{fig:lattice 6}
\end{figure*}

\begin{figure*}
    \centering
    \includegraphics[width=1.0\textwidth]{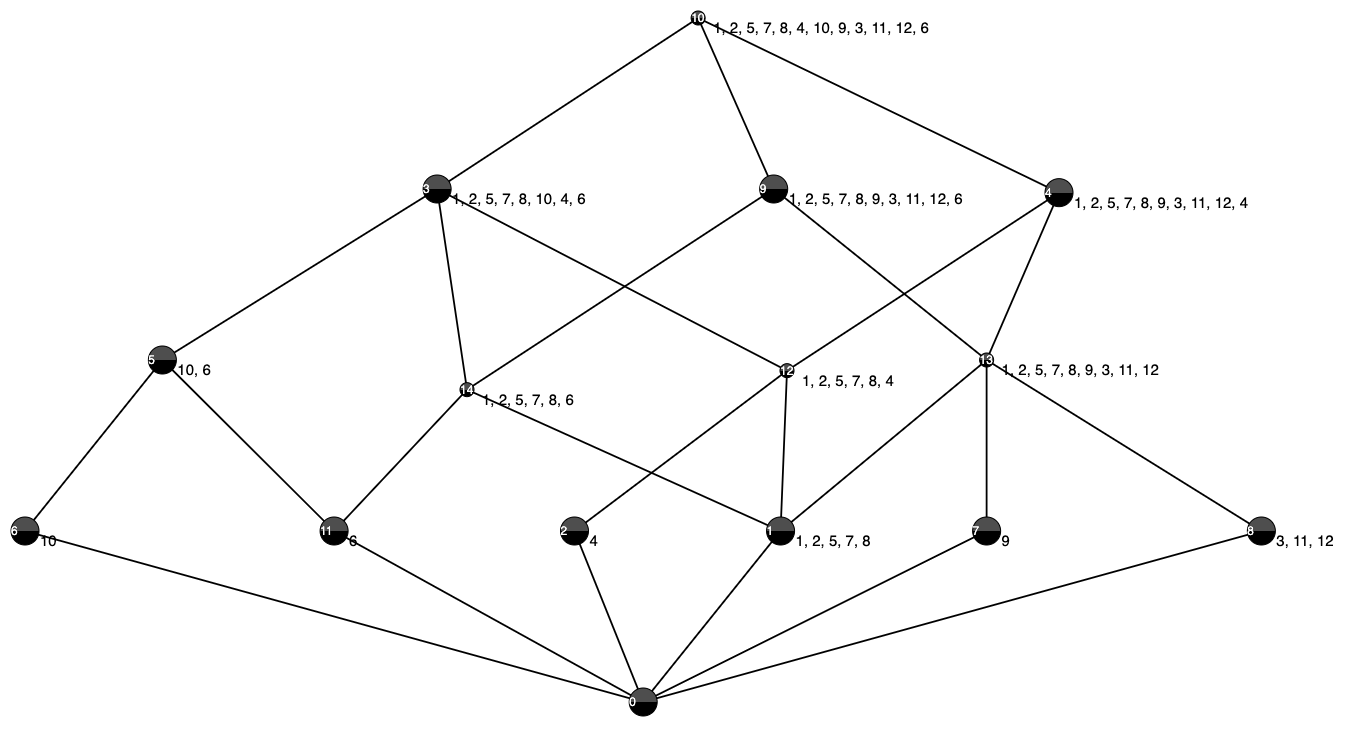}
    \caption{
    This concept lattice is both (1) the 
    most preferred  categorization for the non-crisp agenda $m_3$ of $j_3$ and  (2) the categorization obtained from non-crisp agenda $m_3$ using stability-based method for $\beta=0.5$.}
    \label{fig:lattice 7}
\end{figure*}

\begin{figure*}
    \centering
    \includegraphics [width=1.0\textwidth]{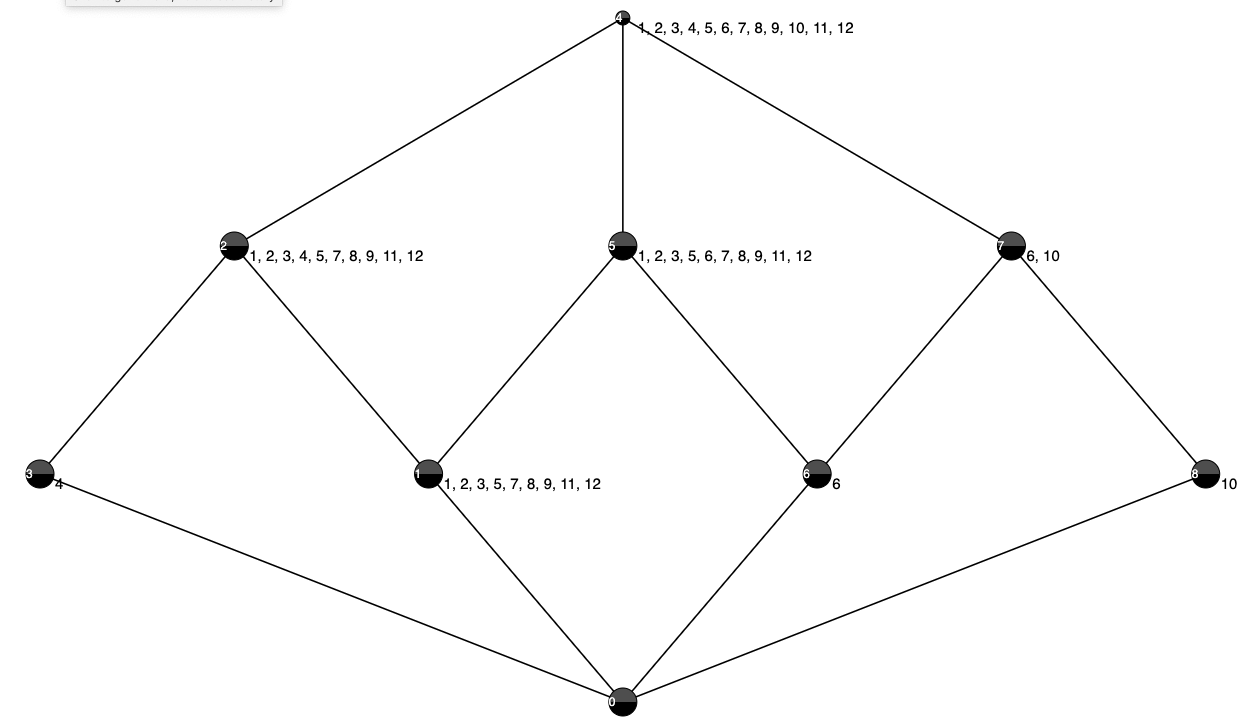}
    \caption{Concept lattice representing  the categorization system associated with  non-crisp agenda  $m_2$ using stability-based method for $\beta=0.5$.}
    \label{fig:lattice 4}
\end{figure*}

\begin{figure*}
    \centering
    \includegraphics[width=1.0\textwidth]{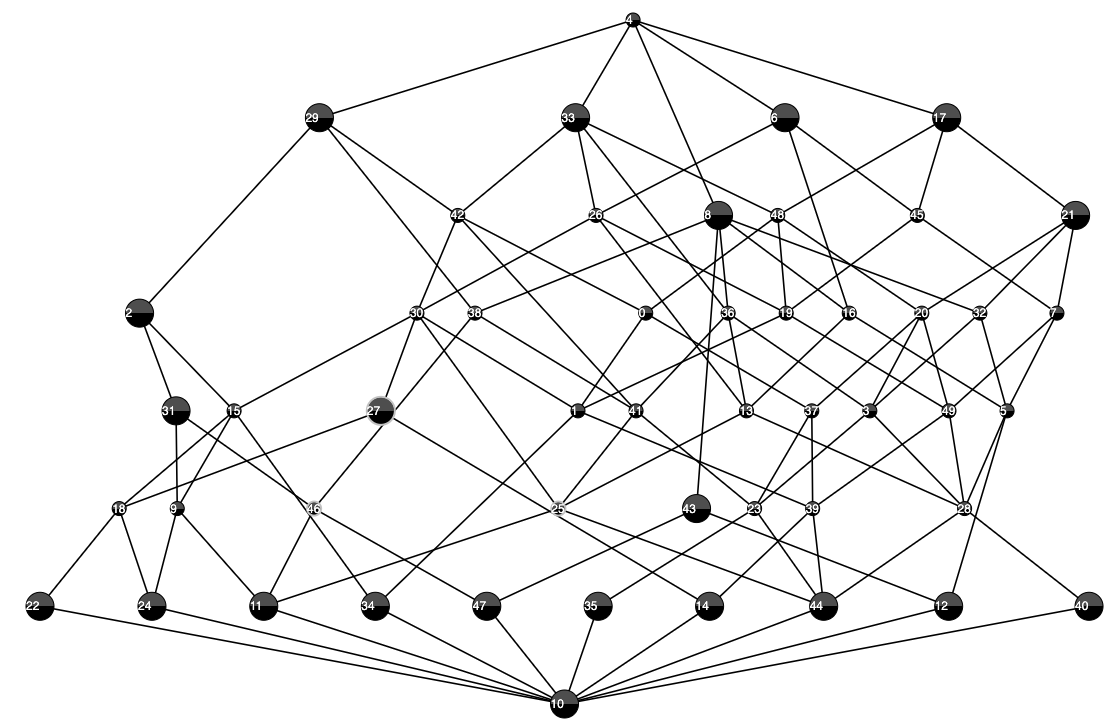}
    \caption{The concept lattice  representing both (1) the categorization obtained from non-crisp agenda $\rhd c$ using stability-based method for $\beta=0.5$ (2) the concept lattice associated with  the (crisp)  agenda consisting of all the features in $X$}
    \label{fig:lattice 10}
\end{figure*}
\end{document}